\documentclass[12pt]{elsarticle}   	
\usepackage{diagbox}
\usepackage{booktabs}
\usepackage{xcolor} 
\usepackage{tikz}
\usepackage{setspace}

\usepackage{algorithm}
\usepackage{algpseudocode}

\usepackage{amsmath}

\usepackage{caption}
\usepackage{subcaption}
\usepackage{array}
\usepackage{mathrsfs}
\usepackage{ctex}
\usepackage{diagbox}
\usepackage[english]{babel}
\usepackage{amsthm}

\usepackage{enumitem}

\usepackage{graphicx}
\usepackage{amssymb}
\usepackage{subcaption}
\usepackage{multirow}

\usepackage{amsthm} 

\usepackage{amsmath}
\allowdisplaybreaks[1]

\usepackage{hyperref}
\usepackage{pifont}
\usepackage{tikz}
\usetikzlibrary{positioning}
\usepackage{tikz}
\usetikzlibrary{shapes.geometric, arrows, patterns}
\usetikzlibrary{arrows, decorations.markings}
\usepackage{graphicx}

\usepackage{textcomp}
\hypersetup{
	colorlinks=true,
}

\newtheorem{proposition}{Proposition}[section]

\newtheorem{example}{Example}[section]

\usepackage{adjustbox}      
\usepackage{booktabs}       
\usepackage{makecell}       
\usepackage{placeins}       
\usepackage{ragged2e}
\usepackage{float} 
\usepackage{geometry}
\usepackage[font=footnotesize,labelfont=bf,labelsep=none]{caption}

\usepackage{bm}

\usepackage{amsmath}
\usepackage{relsize}
\usepackage{tabularx}
\usepackage{amsmath, amsthm}

\newtheorem{definition}{Definition}[section]
\usepackage{booktabs} 
\usepackage{tabularx}
\usepackage{multirow}
\usepackage{amsfonts} 
\usepackage{caption}
\journal{AAA}
\usepackage{rotating}
\usepackage{threeparttable}
\begin{document}
\begin{sloppypar}
\begin{frontmatter}
\title{Feasible strategies for conflict resolution within intuitionistic fuzzy preference-based conflict situations}	

\author[A,B]{Guangming Lang}
\ead{langguangming1984@126.com}
\author[A,B]{Mingchuan Shang}
\ead{shangmingchuan2025@163.com}
\author[C]{Mengjun Hu}
\ead{mengjun.hu@umanitoba.ca}
\author[A,B]{Jie Zhou}
\ead{jie_jpu@163.com}
\author[A,B]{Feng Xu\corref{cor1}}
\cortext[cor1]{Corresponding author.}
\ead{fengxu_csust@163.com}
				
\address[A]{School of Mathematics and Statistics, Changsha University of Science and Technology, Changsha, Hunan, 410114, P.R. China}
\address[B]{Hunan Provincial Key Laboratory of Mathematical Modeling and Analysis in Engineering, Changsha University of Science and Technology, Changsha, Hunan, 410114, P.R. China}
\address[C]{Department of Computer Science, University of Manitoba, Winnipeg, R3T 5V6, Canada}

\begin{abstract}
In three-way conflict analysis, preference-based conflict situations characterize agents' attitudes towards issues by formally modeling their preferences over pairs of issues. However, existing preference-based conflict models rely exclusively on three qualitative relations, namely, preference, converse, and indifference, to describe agents' attitudes towards issue pairs, which significantly limits their capacity in capturing the essence of conflict. To overcome this limitation, we introduce the concept of an intuitionistic fuzzy preference-based conflict situation that captures agents' attitudes towards issue pairs with finer granularity than that afforded by classical preference-based models. Afterwards, we develop intuitionistic fuzzy preference-based conflict measures within this framework, and construct three-way conflict analysis models for trisecting the set of agent pairs, the agent set, and the issue set. Additionally, relative loss functions built on the proposed conflict functions are employed to calculate thresholds for three-way conflict analysis. Finally, we present adjustment mechanism-based feasible strategies that simultaneously account for both adjustment magnitudes and conflict degrees, together with an algorithm for constructing such feasible strategies, and provide an illustrative example to demonstrate the validity and effectiveness of the proposed model.
\end{abstract}
			
\begin{keyword}
Three-way decisions \sep Three-way conflict analysis \sep Feasible strategy \sep Intuitionistic fuzzy preference  
\end{keyword}
\end{frontmatter}

\section{Introduction}
\label{sec:introduction}	

Conflict inevitably arises from divergences in interests and information among individuals and groups, and constitutes a fundamental and pervasive phenomenon in human societies. Beyond its disruptive aspects, conflict frequently plays a constructive role in promoting social change and driving institutional development. As stated by Thomas Crum~\cite{crum1999magic}, conflict constitutes a primary driving force of social change. Owing to its profound influence on individual behavior and collective dynamics, conflict has attracted sustained scholarly attention across a broad range of disciplines, including philosophy~\cite{appiah2020concept} and management science~\cite{huo2023conflict}.
		
Granular computing~\cite{fujita2018resilience,pedrycz2018granular,pawlak1982rough} is a computational paradigm that models, analyzes, and solves complex problems by organizing information into granules. It provides a novel and effective approach for characterizing and quantifying conflict problems. Building on this foundation, many scholars~\cite{dou2024three, suo2022three,wang2023three, li2024dynamic} have sought to develop formal mathematical frameworks for modeling and analyzing conflicts.
For example, Pawlak~\cite{pawlak1984conflicts, pawlak1998inquiry} utilized a three-valued situation table to characterize conflict problems, and built a theoretical framework for conflict analysis.
After that, three-way decisions~\cite{yao2020three,yao2016three,yao2017interval,yao2018three} offers a new perspective for addressing uncertain problems linked to conflict analysis. Building on this paradigm, Yao~\cite{yao2019three} proposed a framework of three-way conflict analysis and extended the three-valued situation table to a multi-valued one. Lang, Miao and Cai~\cite{lang2017three} introduced the concepts of probabilistic conflict, neutral, and allied sets, and computed trisection thresholds using decision-theoretic rough sets. Luo et al.~\cite{luo2022three} defined alliance and conflict functions to address semantic ambiguities arising from aggregation operations. Lang and Yao~\cite{lang2023formal} explored the connection between agent coalitions and issue bundles from the perspective of formal concept analysis. Zhi, Li and Li~\cite{zhi2022multilevel} established a multi-level conflict analysis framework from the perspective of fuzzy formal concept analysis, while Ren et al.~\cite{ ren2024tri} proposed a tri-level conflict analysis model grounded in three-valued formal concept analysis to explore nuanced neutral attitudes.

In three-way conflict analysis, numerous scholars~\cite{li2023three,Chen2025,Zhang2025,liu2025graph} have investigated feasible strategies for conflict resolution from two perspectives: non-adjustment and adjustment. The former seeks to achieve consensus on multiple issues without changing agents' attitudes, while the latter aims to resolve conflicts by adjusting agents' attitudes.
From the non-adjustment perspective, Sun, Ma and Zhao~\cite{sun2016rough} introduced a conflict matrix to analyze the causes of conflicts and identified feasible strategies for conflict resolution. Sun et al.~\cite{sun2020three} established a feasible consensus strategy based on three-way decisions and probabilistic rough sets over two universes. Xu et al.~\cite{xu2022selection} designed a consistency measure to select multiple issues with high consistency as feasible strategies. Liu, Hu and Lang~\cite{liu2025feasible} identified the optimal feasible strategy for a clique of agents based on weighted consistency and inconsistency measures.
Du et al.~\cite{du2022novel} selected issue subsets with low conflict degrees as feasible strategies and employed evaluation functions to identify optimal solutions.
Chen, Xu and Pedrycz~\cite{chen2023conflict} defined threshold stability within a graph model to identify equilibrium states for conflict resolution, and Liu et al.~\cite{liu2025multi} further proposed a graph-based model with multi-attribute evaluation-based mixed stability.
From the adjustment perspective, Tong et al.~\cite{tong2021trust} employed a trust-based mechanism with four feedback conditions to modify agents' attitudes and achieve consensus. 
Gao, Yang and Guo~\cite{gao2025three} proposed conflict resolution methods that adjust preference ratings to form maximal coalitions including as many agents as possible. Jiang, Liu and An~\cite{jiang2025super} developed an optimization model that simultaneously considers group satisfaction and alliance loss in conflict resolution.

Fuzzy sets, introduced by Zadeh~\cite{Zadeh_1965} in 1965, provides a natural and flexible method of modeling vagueness and uncertainty in conflict problems. For example, Li et al.~\cite{li2021conflict} conducted three-way conflict analysis utilizing triangular fuzzy numbers to represent agents' attitudes on issues. Liu and Lin~\cite{liu2015intuitionistic} defined a conflict distance based on intuitionistic fuzzy rough sets. Yi et al.~\cite{yi2021three} and Feng, Yang and Guo~\cite{feng2023three} utilized hesitant fuzzy numbers to characterize agents' attitudes, and investigated three-way conflict analysis models within hesitant fuzzy situation tables. Lang, Miao and Fujita~\cite{lang2019three}, as well as Lang et al.~\cite{ lang2024trisection}, proposed three-way conflict analysis models with Pythagorean fuzzy information. Tang, Chen and Rao~\cite{tang2025distance} introduced fuzzy multiset-based three-way conflict analysis model, while Tang et al.~\cite{tang2025information} further proposed three-way conflict analysis models based on multi-source fuzzy data to overcome the limitations inherent in single-source data-based approaches.
In addition, Yang et al.~\cite{Yang} proposed a three-way conflict analysis model within a hybrid situation table, while Lu, Yang and Guo~\cite{lu2025three} investigated that how to select the optimal scale in multi-scale situation tables.
Although agents' attitudes can be represented in multiple forms, most existing studies predominantly rely on agents' ratings on individual issues, a representation that is inherently limited in its ability to capture the uncertainty and nuance underlying agents' attitudes.
Recently, Hu~\cite{hu2025three} proposed a preference-based three-way conflict analysis model, in which inter-agent relationships includes agreement, partial agreement, and disagreement, and the conflict degree between two agents is determined by three binary relations, namely, preference, converse, and indifference. However, these relations exhibit limited expressive power in capturing nuanced and uncertain attitudes. To overcome this limitation, Lang, Liu, and Hu~\cite{lang2025ija} utilized multi-level preferences to more precisely characterize agents' attitudes towards issue pairs, in contrast to traditional single-level preferences, while Shang and Lang~\cite{shang2025three} adopted fuzzy preferences to represent agents' attitudes towards issue pairs, and developed a three-way conflict analysis model within a fuzzy preference-based conflict situation.
To more comprehensively capture agents' attitudes, we introduce an intuitionistic fuzzy preference-based conflict situation in which agents' preferences are modeled using intuitionistic fuzzy numbers, and develop intuitionistic fuzzy preference-based models for three-way conflict analysis, thereby facilitating a more comprehensive and flexible analysis of agents' attitudes and the development of feasible strategies for conflict resolution.

The main contributions of this study are summarized as follows. First, intuitionistic fuzzy preferences are introduced to describe agents' attitudes towards issue pairs. On this basis, intuitionistic fuzzy preference-based conflict situations and corresponding conflict measures for quantifying conflict degrees are formally defined, and their fundamental properties are investigated. Second, intuitionistic fuzzy preference-based models of three-way conflict analysis are developed to trisect agent pairs, the agent set, and the issue set, and a threshold determination method grounded in decision-theoretic rough sets is proposed. Third, adjustment mechanism-based feasible strategies for conflict resolution are formulated by simultaneously considering adjustment magnitudes and conflict degrees, along with an algorithm for computing the corresponding feasible strategies. An illustrative example is presented to demonstrate the validity and effectiveness of the proposed approach.

The remainder of this study is organized as follows. Section~\ref{sec:Hu's preference-based conflict model} reviews the preference-based conflict analysis model.  Section~\ref{sec:intuitionisticfuzzypreference-basedmodel} introduces the intuitionistic fuzzy preference-based model of three-way conflict analysis. Section~\ref{sec:conflictresolution} provides adjustment mechanism-based feasible strategies for conflict resolution and demonstrates their feasibility and effectiveness through an illustrative example. Section~\ref{sec:Conslusion} concludes this study and outlines directions for future research.

\section{A Review of Preference-Based Three-Way Conflict Analysis Model}
\label{sec:Hu's preference-based conflict model}

This section reviews some concepts of the preference-based three-way conflict analysis model~{\cite{hu2025three}}, including the preference, converse, and indifference relations, the conflict function, and the trisection of agent pairs.

\begin{definition}\label{Hu'spreference-basedsituation}(Hu {\cite{hu2025three}}, 2025)
A preference-based conflict situation is a triplet $PS=(A,I,\{\succ ^I_a\mid a\in A\})$, where $A$ is a finite nonempty set of agents, $I$ is a finite nonempty set of issues, and $\succ^{I}_{a}$ denotes the preference relation of agent $a\in A$ over the issue set $I$, defined as: 
\begin{align}
\succ ^I_a &=\{(i,j)\in I\times I\mid a\ supports\ i\ more\ than\ j \}.
\end{align}
Moreover, the converse relation \( \prec^J_a \) and the indifference relation \( \sim^I_a \)  over the issue set $I$ of agent $a\in A$ are defined by:		
\begin{align}
\prec ^I_a &=\{(i,j)\in I\times I\mid a\ supports\ j\ more\ than\ i \},	\\
		\sim^I_a &= \{(i,j) \in I\times I \mid \neg (i \succ^I_a j) \wedge \neg (j \succ^I_a i)\}.
\end{align}
\end{definition}
		
For two issues $i_1$ and $i_2$, $(i_1, i_2) \in \succ^I_a$ indicates that agent $a$ prefers $i_1$ to $i_2$; $(i_1, i_2) \in \prec^I_a$ indicates that agent $a$ prefers $i_2$ to $i_1$; and $(i_1, i_2) \in \sim^I_a$ indicates that agent $a$ is indifferent between $i_1$ and $i_2$.

\begin{definition}\label{Hu'sconflictfunction}(Hu {\cite{hu2025three}}, 2025)\label{Conflictdegreeonissuepair}
Let $PS=(A,I,\{\succ^I_a\mid a\in A\})$ be a preference-based conflict situation, the conflict function $C_{ij}: A\times A\longrightarrow R$ towards issue pair $(i,j)\in I\times I$ is defined by: for two agents $a,b\in A$,
	\begin{equation}
		C_{ij}(a,b) = 
		\begin{cases}
			\delta^=, & \text{ if } (i\succ_a^I j \land i \succ_b^I j) \lor (i  \sim_a^I j \land i  \sim_b^I j) \lor (j\succ_a^I i \land j \succ_b^I i), \\
			\delta^\approx, & \text{ if } i  \sim_a^I j \oplus i  \sim_b^I j,\\
			\delta^\asymp, &  \text{ if } (i\succ_a^I j \land j \succ_b^I i) \lor (j \succ_a^I i \land i \succ_b^I j),
		\end{cases}
	\end{equation}
where $\oplus$ denotes logical exclusive or. 
\end{definition}

The three values $\delta^{=}$, $\delta^{\approx}$, and $\delta^{\asymp}$ satisfy \(\delta^{=} < \delta^{\approx} < \delta^{\asymp}\) and \(\delta^{=} \leqslant 0 < \delta^{\asymp}\), where $\delta^{=}$, $\delta^{\approx}$, and $\delta^{\asymp}$ represent the conflict degrees corresponding to agreement, partial agreement, and disagreement, respectively. The detailed specifications of conflict degrees towards issue pairs are summarized in Table~\ref{table:threerelationshipbetweentwoagents}.		
		
\begin{table}[H]
	\centering
	\caption{The conflict degree between two agents.}
	\label{table:threerelationshipbetweentwoagents}
	\setlength{\tabcolsep}{10pt}  
	\adjustbox{width=\textwidth}{ 
		\begin{tabular}{llc}
			\toprule
			$\mathrm{Relationship}$ & Preferences of agent $a$ and agent $b$ & $ {C_{ij}(a,b) }$ \\
			\midrule
			Agreement & $(i\succ_a^I j \land i \succ_b^I j) \lor (i  \sim_a^I j \land i  \sim_b^I j) \lor (j\succ_a^I i \land j \succ_b^I i)$ & $\delta^=$ \\
			Partial agreement & $i  \sim_a^I j \oplus i  \sim_b^I j$ & $\delta^{\approx}$  \\
			Disagreement & $(i\succ_a^I j \land j \succ_b^I i) \lor (j \succ_a^I i \land i \succ_b^I j)$ & $\delta^{\asymp}$  \\
			\bottomrule
		\end{tabular}
	}
	\\[5pt] 
	\small  
	\raggedright 
\end{table}

\begin{definition}\label{Hu's conflictfunction}(Hu {\cite{hu2025three}}, 2025)\label{Conflictdegreeonissueset}
Let $PS=(A,I,\{\succ ^I_a\mid a\in A\})$ be a preference-based conflict situation, the conflict function $C_I: A\times A\longrightarrow R$ towards the issue set I is defined by: for two agents $a,b\in A$,
\begin{align}
C_I(a,b)&=\delta^= ( \left| \succ_a^I \cap \succ_b^I \right| + \frac{\left| \sim_a^I \cap \sim_b^I \right| + \left| I \right|}{2} ) 
+ \delta^\approx ( \left|\succ_a^I \cap \sim_b^I \right| + \left| \sim_a^I \cap\succ_b^I \right| )\notag \\&\hspace{1.0em} + \delta^\asymp ( \left| \succ_a^I \cap 	\prec _b^I \right| ),
\end{align}
where $|\cdot |$ denotes the cardinality of a set.
\end{definition}

For two agents $a,b\in A$, the value $C_I(a,b)$ quantifies the conflict degree between two agents $a$ and $b$ with respect to the issue set $I$. Moreover, Hu~\cite{hu2025three} utilized the max-min normalization to standardize this conflict degree given by Definition~\ref{Hu'sconflictfunction} to the interval $[0,1]$.

\begin{definition}\label{Hu'sScoflictfunction}(Hu {\cite{hu2025three}}, 2025)\label{D3}
Let $PS=(A,I,\{\succ ^I_a\mid a\in A\})$ be a preference-based conflict situation,  the normalized conflict function $NC_I: A\times A\longrightarrow[0,1] $ towards the issue set I is defined by: for two agents $a,b\in A$,
\begin{align}
NC_I(a,b)=\frac{C_I(a,b)-min\ C_I}{max\ C_I-min\ C_I},
\end{align}
where $min\ C_I=\frac{\delta^{=}|I|(|I|+1)}{2}$ and $max\ C_I=\frac{\delta^{\asymp}|I|(|I|+1)}{2}$.
\end{definition}

Based on the normalized conflict function, Hu~\cite{hu2025three} scaled the conflict degree between two agents to the interval $[0,1]$ and subsequently categorized all agent pairs into alliance, neutrality, and conflict relations with respect to the set of issues.

\begin{definition}\label{Hu'strisection}(Hu {\cite{hu2025three}}, 2025)
Let $PS=(A,I,\{\succ ^I_a\mid a\in A\})$ be a preference-based conflict situation, and a pair of thresholds $(\zeta_*,\zeta^*)$ such that $0\leqslant\zeta_*<\zeta^*\leqslant1$. Then the alliance relation $\mathbb{R}^=_I$, neutrality relation $\mathbb{R}^\approx_I$, and conflict relation $\mathbb{R}_{I}^\asymp$ on A with respect to the issue set $I$ are defined by:
\begin{align}
\mathbb{R}^{=}_I &= \{(a,b) \in A \times A \mid C_I^*(a,b) < \zeta_*\},\\
\mathbb{R}^{\approx}_I &= \{(a,b) \in A \times A \mid \zeta_* \leq C_I^*(a,b) \leq \zeta^*\},\\
\mathbb{R}^\asymp_I &= \{(a,b) \in A \times A \mid C_I^*(a,b) > \zeta^*\}.
\end{align}
\end{definition}

In Hu's work, the conflict degree between two agents is inferred from their preference relations. Nevertheless, the presence of uncertainty in agents'  preferences may result in inconsistent or unstable outcomes for three-way conflict analysis, and the topic that how to accurately represent their preferences has not been thoroughly explored in the existing literature. In what follows, we will characterize agents' attitudes towards issue pairs using intuitionistic fuzzy preferences and extend Hu's preference-based conflict analysis model accordingly.

\section{Intuitionistic Fuzzy Preference-Based Model of Three-Way Conflict Analysis}	\label{sec:intuitionisticfuzzypreference-basedmodel}

In this section, we characterize   agents' attitudes towards issue pairs through intuitionistic fuzzy preferences, and provide intuitionistic fuzzy preference-based models of three-way conflict analysis.

\subsection{Intuitionistic Fuzzy Preference-Based Conflict Situation}

We first introduce an intuitionistic fuzzy preference-based conflict situation.

\begin{definition}\label{IFPCS}
An intuitionistic fuzzy preference-based conflict situation is  defined by a triplet $IFPS=(A,I,\{\widetilde{R}_a\mid a\in A\})$, where $A=\{a_1,a_2,\dots,a_n\}$ is a set of agents, $I=\{i_1,i_2,\dots,i_m\}$ is a set of issues, $\widetilde{R}_a=(\widetilde{r}_a(i,j))_{m\times m}$ is the set of intuitionistic fuzzy preferences of agent $a\in A$ towards all pairs of issues, in which 
\begin{itemize}
\setlength{\itemsep}{0pt} %
\item $\widetilde{r}_a(i,j)=(\mu_a(i,j),\nu_a(i,j))$ is an intuitionistic fuzzy number; 
\item $\mu_a(i,j)=\nu_a(j,i)$ and $\mu_{a}(i,i)=\nu_{a}(i,i)=0.5$.
\end{itemize}	
\end{definition}
	
For the first item, $\mu_a(i,j)$ denotes the degree to which an agent prefers issue 
$i$ to issue $j$, and $\nu_a(i,j)$ represents the degree of non-preference (or converse preference).
For the second item, the preference degree of agent $a$ on issue pair $(i,j)$ is equal to the non-preference degree of agent $a$ on issue pair $(j,i)$, and $\mu_{a}(i,i)=\nu_{a}(i,i)=0.5$, which aligns with human cognition. In addition, the value $\pi_a(i,j)=1-\mu_a(i,j)-\nu_a(i,j)$ stands for the preference hesitation degree of agent $a$ towards issue pair $(i,j)$.
For convenience, we assume that $(\widetilde{R}_{a}^{\mu})^\mathrm{T}=((\mu_a(i,j))_{m\times m})^\mathrm{T}=(\nu_a(i,j))_{m\times m}=\widetilde{R}_{a}^{\nu}$, where $(\widetilde{R}_{a}^{\mu})^\mathrm{T}$ denotes the transpose of $\widetilde{R}_{a}^{\mu}$. Accordingly, 
$\widetilde{R}_{a}$ can be conveyed separately through $\widetilde{R}_{a}^{\mu}$ and $\widetilde{R}_{a}^{\nu}$. Therefore, we can utilize $(A,I,\{{\widetilde{R}}^{\mu}_a\mid a\in A\})$ to represent an intuitionistic fuzzy preference-based conflict situation.

\begin{proposition}\label{T1}
Let $IFPS=(A,I,\{\widetilde{R}_a\mid a\in A\})$ be an intuitionistic fuzzy preference-based conflict situation. For an agent $a\in A$, and two issues $i,j\in I$, it holds that $0\leqslant\mu_{a}(i,j)+\mu_{a}(j,i)\leqslant 1$ and $0\leqslant\nu_{a}(i,j)+\nu_{a}(j,i)\leqslant 1$.
\end{proposition}
\begin{proof}
For  $0\leqslant\mu_{a}(i,j)+\nu_{a}(i,j)\leqslant 1$ and  $\nu_a(i,j) = \mu_a(j,i)$, it is obvious that
$ 0 \leqslant \mu_a(i,j) + \mu_a(j,i) \leqslant 1$ and $0 \leqslant \nu_a(i,j) + \nu_a(j,i) \leqslant 1$.
\end{proof}

\begin{example}\label{example1}
We illustrate the Middle East Conflict by the intuitionistic fuzzy preference-based conflict situation $(A,I, \{\widetilde{R}_{a_1}, \widetilde{R}_{a_2}, \widetilde{R}_{a_3}, \widetilde{R}_{a_4}, \widetilde{R}_{a_5},\widetilde{R}_{a_6}\})$, where
$A=\{a_1,a_2,a_3,a_4,a_5,a_6\}$ represents Israel, Egypt, Palestine, Jordan, Syria, and Saudi Arabia, respectively, $I=\{i_1,i_2,i_3,i_4,i_5\}$ represents ``Autonomous Palestinian state on the West Bank and Gaza'', ``Israeli military outpost along the Jordan River'', ``Israel retains East Jerusalem'', ``Israeli military outposts on the Golan Heights'', and ``Arab countries grant citizenship to Palestinians who choose to remain within their borders'', respectively; $\widetilde{R}_{a_1}, \widetilde{R}_{a_2}, \widetilde{R}_{a_3}, \widetilde{R}_{a_4}$, and $ \widetilde{R}_{a_5}$ stand for the intuitionistic fuzzy preferences of agents $a_1, a_2, a_3, a_4$, and $a_5$, respectively, on all pairs of issues.
For example, the intuitionistic fuzzy preference $\widetilde{R}_{a_1}$ is presented as follows.
\[
\widetilde{R}_{a_1} = \begin{bmatrix}
				(0.50,0.50) & (0.90,0.03) & (0.93,0.05) & (0.89,0.06) & (0.97,0.02) \\
				(0.03,0.90) & (0.50,0.50) & (0.95,0.04) & (0.90,0.08) & (0.93,0.05) \\
				(0.05,0.93) & (0.04,0.95) & (0.50,0.50) & (0.96,0.03) & (0.90,0.07) \\
				(0.06,0.89) & (0.08,0.90) & (0.03,0.96) & (0.50,0.50) & (0.95,0.01) \\
				(0.02,0.97) & (0.05,0.93) & (0.07,0.90) & (0.01,0.95) & (0.50,0.50)
			\end{bmatrix}.
			\]
For agent $a_1$, $\widetilde{r}_{a_1}(i_1,i_3)=(0.93, 0.05)$ indicates that the preference degree on issue pair $(i_1,i_3)$ is $0.93$, the non-preference degree on issue pair $(i_1,i_3)$ is $0.05$, and the preference hesitation degree on issue pair $(i_1,i_3)$ is $0.02$.
For convenience, the intuitionistic fuzzy preferences are abbreviated as $\widetilde{R}_{a_1}^{\mu}$, $\widetilde{R}_{a_2}^{\mu}$, $\widetilde{R}_{a_3}^{\mu}$, $\widetilde{R}_{a_4}^{\mu}$, $\widetilde{R}_{a_5}^{\mu}$, and $\widetilde{R}_{a_6}^{\mu}$, respectively, as follows.			
{\normalsize
\setlength{\arraycolsep}{1.0pt}
\[
\begin{array}{cc}
\widetilde{R}_{a_1}^{\mu} =
\begin{bmatrix}
0.50 & 0.90 & 0.93 & 0.89 & 0.97 \\
						0.03 & 0.50 & 0.95 & 0.90 & 0.93 \\
						0.05 & 0.04 & 0.50 & 0.96 & 0.90 \\
						0.06 & 0.08 & 0.03 & 0.50 & 0.95 \\
						0.02 & 0.05 & 0.07 & 0.01 & 0.50
					\end{bmatrix},
					& 
					\widetilde{R}_{a_2}^{\mu} = 
					\begin{bmatrix}
						0.50 & 0.02 & 0.04 & 0.06 & 0.03 \\
						0.97 & 0.50 & 0.05 & 0.07 & 0.04 \\
						0.95 & 0.93 & 0.50 & 0.02 & 0.05 \\
						0.90 & 0.91 & 0.96 & 0.50 & 0.03 \\
						0.92 & 0.90 & 0.89 & 0.95 & 0.50
					\end{bmatrix},
				\end{array}
				\]
				\[
				\begin{array}{cc}
					\widetilde{R}_{a_3}^{\mu} = 
					\begin{bmatrix}
						0.50 & 0.10 & 0.88 & 0.12 & 0.85 \\
						0.89 & 0.50 & 0.10 & 0.09 & 0.13 \\
						0.10 & 0.89 & 0.50 & 0.10 & 0.08 \\
						0.86 & 0.90 & 0.87 & 0.50 & 0.15 \\
						0.10 & 0.07 & 0.17 & 0.05 & 0.50
					\end{bmatrix},
				&
					\widetilde{R}_{a_4}^{\mu} = 
					\begin{bmatrix}
						0.50 & 0.85 & 0.10 & 0.80 & 0.10 \\
						0.12 & 0.50 & 0.82 & 0.10 & 0.79 \\
						0.88 & 0.17 & 0.50 & 0.81 & 0.13 \\
						0.11 & 0.86 & 0.09 & 0.50 & 0.84 \\
						0.83 & 0.19 & 0.82 & 0.10 & 0.50
					\end{bmatrix},
				\end{array}
				\]
				\[
				\begin{array}{cc}
					\widetilde{R}_{a_5}^{\mu} = 
					\begin{bmatrix}
						0.50 & 0.45 & 0.40 & 0.42 & 0.38 \\
						0.48 & 0.50 & 0.43 & 0.39 & 0.37 \\
						0.46 & 0.47 & 0.50 & 0.41 & 0.36 \\
						0.44 & 0.49 & 0.48 & 0.50 & 0.35 \\
						0.43 & 0.44 & 0.45 & 0.46 & 0.50
					\end{bmatrix},
					&
					\widetilde{R}_{a_6}^{\mu} = 
					\begin{bmatrix}
						0.50 & 0.90 & 0.05 & 0.85 & 0.06 \\
						0.07 & 0.50 & 0.04 & 0.88 & 0.03 \\
						0.90 & 0.90 & 0.50 & 0.02 & 0.01 \\
						0.11 & 0.10 & 0.90 & 0.50 & 0.09 \\
						0.90 & 0.08 & 0.93 & 0.90 & 0.50
					\end{bmatrix}.
				\end{array}
				\]
			}
\end{example}

\subsection{Intuitionistic Fuzzy Preference-Based Conflict Function}\label{IFPCF}
Subsequently, we introduce an intuitionistic fuzzy preference-based conflict function to characterize the relationship between two agents towards an issue pair.

\begin{definition}\label{Conflictfunction}
Let $IFPS=(A,I,\{	\widetilde{R}_a\mid a\in A\})$ be an intuitionistic fuzzy preference-based conflict situation. Then the conflict function $CF_{ij}: A\times A\longrightarrow [0,1]$ towards issue pair $(i,j)\in I\times I$ is defined by: for two agents $a,b\in A$, 
\begin{align*}
CF_{ij}(a,b)= \frac{1}{2} ( |\mu_a(i,j)-\mu_b(i,j)| + |\nu_a(i,j)-\nu_b(i,j)|+ |\pi_a(i,j)-\pi_b(i,j)|).
\end{align*}
\end{definition}

The conflict function given Definition~\ref{Conflictfunction} ultilizes the distance between two intuitionistic fuzzy numbers to measure the conflict degree between two agents with respect to an issue pair.

\begin{proposition}\label{pro1}
Let $IFPS=(A,I,\{\widetilde{R}_a\mid a\in A\})$ be an intuitionistic fuzzy preference-based conflict situation. Then the conflict function $CF_{ij}(\cdot,\cdot)$ satisfies the following properties: for three agents $ a,b,c\in A$, and two issues $ i,j\in I$, 
\begin{enumerate}[label={(\arabic*)}]
\item Non-negativity: $CF_{ij}(a,b)\geqslant0$, and $CF_{ij}(a,a)=0$;
\item Symmetry:  $CF_{ij}(a,b)=CF_{ij}(b,a)$;
\item Triangle inequality: $CF_{ij}(a,b)+CF_{ij}(b,c)\geqslant CF_{ij}(a,c)$.
\end{enumerate}	
\end{proposition}
\begin{proof}
The proofs of $(1)$, $(2)$, and $(3)$ are straightforward.\qedhere
\end{proof}
		
\begin{proposition}
Let $IFPS=(A,I,\{\widetilde{R}_a\mid a\in A\})$ be an intuitionistic fuzzy preference-based conflict situation. Then we have $0\leqslant CF_{ij}(a,b)\leqslant 1$ for two agents $ a,b\in A$ towards issue pair $(i,j)\in I\times I$.
\end{proposition}
\begin{proof}
By Proposition~\ref{pro1}, we have $CF_{ij}(a,b)\geqslant 0$ for two agents $ a,b\in A$ towards issues $ i,j\in I$. It remains to prove that $CF_{ij}(a,b)\leqslant 1$.
First, for two agents $a,b\in A$ and two issues $i,j\in I$, we have
\begin{align}
\pi_a(i,j)-\pi_b(i,j)&= (1-\mu_a(i,j)-\nu_a(i,j))-(1-\mu_b(i,j)-\nu_b(i,j)) \notag \\&= (\mu_b(i,j)-\mu_a(i,j)) + (\nu_b(i,j)-\nu_a(i,j)).
\end{align}
After that, suppose $\Delta\mu = \mu_a(i,j)-\mu_b(i,j)$ and $
\Delta\nu = \nu_a(i,j)-\nu_b(i,j)$, we obtain
\[
CF_{ij}(a,b)=\frac{1}{2}\bigl(|\Delta\mu| + |\Delta\nu| + |\Delta\mu+\Delta\nu|\bigr).
\]
Since $0\leqslant\mu_{a}(i,j)+\nu_{a}(i,j)\leqslant 1$, and $0\leqslant\mu_{b}(i,j)+\nu_{b}(i,j)\leqslant 1$, then we get $|\Delta\mu+\Delta\nu|\leq1,\ |\Delta\mu|\leq1$, and $ |\Delta\nu|\leq1$. Third,
if $\Delta\mu\cdot\Delta\nu\geqslant0$,
then $|\Delta\mu|+|\Delta\nu|=|\Delta\mu+\Delta\nu|$. Hence,
\[			
CF_{ij}(a,b)= \frac{1}{2}\bigl(2*|\Delta\mu+\Delta\nu|\bigr)\leq 1.
\]
If $\Delta\mu\cdot\Delta\nu<0$,
then $|\Delta\mu+\Delta\nu|=|\;|\Delta\mu|-|\Delta\nu||$. Hence,
\[CF_{ij}(a,b)= \frac{1}{2}\bigl(|\Delta\mu| + |\Delta\nu|+|\;|\Delta\mu|-|\Delta\nu||\bigr)\leq \frac{1}{2}\cdot max\{2|\Delta\mu|, 2|\Delta\nu|\}\leq 1.\]
Therefore, it concludes that $0\leqslant CF_{ij}(a,b)\leqslant 1$.\qedhere
\end{proof}

Building on the conflict function defined towards issue pairs, we construct an agent-level conflict function with respect to multiple issues for three-way conflict analysis.
		
\begin{definition}\label{D16}
Let $IFPS=(A,I,\{\widetilde{R}_a\mid a\in A\})$ be an intuitionistic fuzzy preference-based conflict situation. Then the conflict function $CF_{J}: A\times A\longrightarrow [0,1]$ towards an issue bundle $J\subseteq I$ $(|J|\geqslant2)$ is defined by: for two agents $a,b\in A$,
\begin{align}
CF_{J}(a,b) = \frac{1}{|J|(|J|-1)} \sum_{\substack{(i,j) \in J\wedge i\neq j}} CF_{ij}(a,b).
\end{align}
\end{definition}

The value $CF_{J}(a,b)$ quantifies the conflict degree between two agents $a$ and $b$ with respect to a bundle of issues $J\subseteq I$. That is, the conflict function $CF_{J}(\cdot,\cdot)$ provides a formal measure of inter-agent conflict over the specified issue bundle.

\begin{proposition}\label{Conflictfunctiontowardsmultipleissues}
Let $IFPS=(A,I,\{\widetilde{R}_a\mid a\in A\})$ be an intuitionistic fuzzy preference-based conflict situation. Then the conflict function $CF_J$ towards a bundle of issues $J\subseteq I$ satisfies the following properties: for three agents $ a,b,c\in A$, 
\begin{enumerate}[label={(\arabic*)}]
\item Non-negativity: $0\leqslant CF_{J}(a,b)\leqslant1$, and $CF_{J}(a,a)=0$;
\item Symmetry:  $CF_{J}(a,b)=CF_{J}(b,a)$;
\item Triangle inequality: $CF_{J}(a,b)+CF_{J}(b,c)\geqslant CF_{J}(a,c)$.
\end{enumerate}	
\end{proposition}
\begin{proof}
The proofs of $(1)$, $(2)$, and $(3)$ are straightforward.\qedhere
\end{proof}

\begin{table}[htp]
\centering
\begin{minipage}{0.48\textwidth}
\centering
\renewcommand{\arraystretch}{0.9}
\caption{The conflict degree between two agents $a_2$ and $a_4$.}
\label{table2}
\setlength{\tabcolsep}{6pt}
\begin{tabular}{llllll}
\toprule
\textbf{$I\backslash I$} & $i_1$ & $i_2$ & $i_3$ & $i_4$ & $i_5$ \\
\midrule
$i_1$ & 0.00 & 0.85 & 0.07 & 0.79 & 0.09 \\
$i_2$ &      & 0.00 & 0.77 & 0.05 & 0.75 \\
$i_3$ &      &      & 0.00 & 0.87 & 0.08 \\
$i_4$ &      &      &      & 0.00 & 0.85 \\
$i_5$ &      &      &      &      & 0.00 \\
\bottomrule
\end{tabular}
\end{minipage}
\hfill
\begin{minipage}{0.48\textwidth}
\centering
\renewcommand{\arraystretch}{0.9}
\caption{The conflict degree between two agents $a_3$ and $a_4$.}
\label{table3} 
\setlength{\tabcolsep}{6pt}
\begin{tabular}{llllll}
\toprule
\textbf{$I\backslash I$} & $i_1$ & $i_2$ & $i_3$ & $i_4$ & $i_5$ \\
\midrule
$i_1$ & 0.00 & 0.77 & 0.78 & 0.75 & 0.75 \\
$i_2$ &      & 0.00 & 0.72 & 0.04 & 0.78 \\
$i_3$ &      &      & 0.00 & 0.78 & 0.70 \\
$i_4$ &      &      &      & 0.00 & 0.74 \\
$i_5$ &      &      &      &      & 0.00 \\
\bottomrule
\end{tabular}
\end{minipage}
\end{table}
		
\begin{example}\label{example3}
(Continued from Example  \ref{example1}). According to Definition~\ref{Conflictfunction}, we have
\begin{align}
CF_{i_2i_1}(a_3,a_4)=&\frac{1}{2}\times(|0.89-0.12|+|0.10-0.85|+|0.01-0.03|)=0.77,\\
CF_{i_2i_1}(a_2,a_4)=&\frac{1}{2}\times(|0.97-0.12|+|0.02-0.85|+|0.01-0.03|)=0.85.
\end{align}
Similarly, we compute the conflict degree between two agents towards all issue pairs and show these results by Tables~\ref{table2} and~\ref{table3}. Subsequently,  
by Definition~\ref{D16}, we have
\begin{align}
CF_I(a_4,a_5)=&\frac{2}{20}\times(0.40+0.39+0.38+0.38+0.39+0.37+0.42+0.40\notag \\&+0.37+0.49)\notag \\=& 0.40.
\end{align}
Similarly, we 
calculate the conflict degree between the other two agents towards the issue set $I$ and show these results by Table~\ref{table_dI}.
\end{example}

\begin{table}[htp]
\centering
\caption{The conflict degree between two agents with respect to $I$.}
\label{table_dI}
\setlength{\tabcolsep}{10pt} 
\renewcommand{\arraystretch}{1.2}
\begin{tabularx}{\textwidth}{c *{6}{>{\centering\arraybackslash}X}}
\toprule
\textbf{$A\backslash A$} & $a_1$ & $a_2$ & $a_3$ & $a_4$ & $a_5$ & $a_6$ \\
\midrule
$a_1$ & 0.00 & 0.90 & 0.68 & 0.40 & 0.53 & 0.64 \\
$a_2$ & & 0.00 & 0.44 & 0.52 & 0.47 & 0.36 \\
$a_3$ &  &  & 0.00 & 0.68 & 0.48 & 0.59 \\
$a_4$ &  &  &  & 0.00 & 0.40 & 0.44 \\
$a_5$ &  &  &  &  & 0.00 & 0.48 \\
$a_6$ &  &  &  &  &  & 0.00 \\
\bottomrule
\end{tabularx}
\end{table}
			
The intuitionistic fuzzy preference provides a more nuanced characterization of agents' attitudes towards issue pairs by simultaneously capturing two complementary perspectives, and the intuitionistic fuzzy preference-based conflict function offers a more accurate and expressive representation of the relationship between two agents.

\subsection{Intuitionistic Fuzzy Preference-Based Trisections of Agent Pairs, the Agent Set and the Issue Set}\label{Conflictfunction-basedTrisection}

This section presents intuitionistic fuzzy preference-based trisections of agent pairs, the agent set, and the issue set.

\begin{definition}\label{SC}
Let $IFPS=(A,I,\{\widetilde{R}_a \mid a \in A\})$ be an intuitionistic fuzzy preference-based conflict situation, and 
a pair of thresholds $(\alpha_*,\alpha^*)$ such that $0\leqslant\alpha_*<\alpha^*\leqslant1$. Then the conflict relation $\mathbb{IR}_{J}^{\asymp}$, neutrality relation $\mathbb{IR}_{J}^{\approx}$, and alliance relation $\mathbb{IR}_{J}^{=}$ with respect to a bundle of issues $J\subseteq I$ are defined by:
\begin{enumerate}[label={(\arabic*)}]
\item $\mathbb{IR}_{J}^{\asymp}=\{(a,b)\in A\times A\mid CF_{J}(a,b)\geqslant \alpha^*\}$;
\item $\mathbb{IR}_{J}^{\approx}=\{(a,b)\in A\times A\mid \alpha_*<CF_{J}(a,b)<\alpha^*\}$;
\item$\mathbb{IR}_{J}^{=}=\{(a,b)\in A\times A\mid CF_{J}(a,b)\leqslant \alpha_*\}$.
\end{enumerate}	
\end{definition}

Based on the conflict degree between two agents, the set of agent pairs is partitioned into three categories with respect to multiple issues: the conflict relation, the neutrality relation, and the alliance relation.

\begin{definition}\label{Trisectionofagentsbasedonrelation}
Let $IFPS=(A,I,\{\widetilde{R}_a \mid a \in A\})$ be an intuitionistic fuzzy preference-based conflict situation, and 
a pair of thresholds $(\alpha_*,\alpha^*)$ such that $0\leqslant\alpha_*<\alpha^*\leqslant1$. Then the conflict coalition $\mathbb{IR}_{J}^{\asymp}(a)$, neutrality coalition $\mathbb{IR}_{J}^{\approx}(a)$, and alliance coalition $\mathbb{IR}_{J}^{=}(a)$ of agent $a\in A$ with respect to a bundle of issues $J\subseteq I$ are defined by:
\begin{enumerate}[label={(\arabic*)}]
	\item $\mathbb{IR}_{J}^{\asymp}(a)=\{b\in A\mid CF_{J}(a,b)\geqslant \alpha^*\}$;
	\item $\mathbb{IR}_{J}^{\approx}(a)=\{b\in A\mid \alpha_*<CF_{J}(a,b)<\alpha^*\}$;
	\item$\mathbb{IR}_{J}^{=}(a)=\{b\in A\mid CF_{J}(a,b)\leqslant \alpha_*\}$.
\end{enumerate}	
\end{definition}

Based on the conflict degree between two agents, the agent set is categorized into three disjoint parts: conflict coalition, neutrality coalition, and alliance coalition of an agent towards multiple issues.

\begin{proposition}
Let $IFPS=(A,I,\{\widetilde{R}_a \mid a \in A\})$ be an intuitionistic fuzzy preference-based conflict situation. 
Then we have the following results: for two agents $a,b\in A$ and a bundle of issues $J\subseteq I$,
\begin{enumerate}[label={(\arabic*)}]
\item $b\in \mathbb{IR}_{J}^{\asymp}(a)\iff a\in \mathbb{IR}_{J}^{\asymp}(b)$;
\item $b\in \mathbb{IR}_{J}^{\approx}(a)\iff a\in \mathbb{IR}_{J}^{\approx}(b)$;
\item$b\in \mathbb{IR}_{J}^{=}(a)\iff a\in \mathbb{IR}_{J}^{=}(b)$.
\end{enumerate}		
\end{proposition}

\begin{example}
(Continued from Example  \ref{example3}). By taking the thresholds $\alpha^*=0.6$ and $\alpha_*=0.4$, according to Definition~\ref{SC}, we have the following results:
\begin{align}
\mathbb{IR}_{I}^{\asymp}(a_1) &= \{b\in A \mid CF_{I}(a_1,b)\geqslant 0.6\} = \{a_2,a_3,a_6\}, \\
\mathbb{IR}_{I}^{\approx}(a_1) &= \{b\in A \mid 0.4 < CF_{I}(a_1,b) < 0.6\} = \{a_5\},\\
\mathbb{IR}_{I}^{=}(a_1) &= \{b\in A \mid CF_{I}(a_1,b)\leqslant 0.4\} = \{a_1,a_4\}.
\end{align}
Similarly, we calculate the conflict, neutrality, and alliance coalitions of $a_2, a_3,a_3,a_4$, and $a_5$ towards the whole set $I$ and show these results in Table~\ref{tab:trisection 1}.
\end{example}	

\begin{table}[htp]
\centering
\caption{The conflict, neutrality, and alliance coalitions towards  $I$ ($\alpha^*=0.6$, $\alpha_*=0.4$).}
\label{tab:trisection 1}
\setlength{\tabcolsep}{30pt} 
\renewcommand{\arraystretch}{1.2}
\begin{tabularx}{\textwidth}{llll}
\toprule
$A$ & $\mathbb{IR}_{I}^{\asymp}(a)$ & $\mathbb{IR}_{I}^{\approx}(a)$ & $\mathbb{IR}_{I}^{=}(a)$ \\
\midrule
$a_1$ & $\{a_2,a_3,a_6\}$ & $\{a_5\} $ & $\{a_1,a_4\} $ \\
$a_2$ & $\{a_1\}$ & $\{a_3,a_4,a_5\}$ & $\{a_2, a_6\}$ \\
$a_3$ & $\{a_1,a_4\}$ & $\{a_2,a_5,a_6\} $ & $\{a_3\}$ \\
$a_4$ & $\{a_3\}$ & $\{a_2,a_6\}$ & $\{a_1,a_4, a_5\}$ \\
$a_5$ & $\emptyset$ & $\{a_1,a_2,a_3,a_6\}$ & $\{ a_4,a_5\}$ \\
$a_6$ & $\{a_1\}$ & $\{ a_3,a_4,a_5\}$ & $\{a_2,a_6\}$ \\
\bottomrule
\end{tabularx}
\end{table}
\begin{definition}\label{D19}
Let $IFPS=(A,I,\{\widetilde{R}_a \mid a \in A\})$ be an intuitionistic fuzzy preference-based conflict situation. 
Then the conflict measure $CM: A\times 2^I\longrightarrow [0,1]$ with respect to a bundle of issues $J\subseteq I$ is defined by: for an agent $a\in A$,
\begin{align}
CM(a,J) =
\begin{cases} 
\frac{1}{|A|-1}\sum\limits_{b\in A} CF_J(a,b), &  |A| > 1, \\
0, & |A| = 1.
\end{cases}
\end{align}
\end{definition}

The conflict measure quantifies the conflict degree between a given agent and all other agents with respect to multiple issues. In particular, when \(J = I\), the conflict measure \(CM(a, I)\) represents the conflict degree between agent \(a\) and all other agents towards the whole issues.

\begin{definition}\label{D20}
Let $IFPS=(A,I,\{\widetilde{R}_a \mid a \in A\})$ be an intuitionistic fuzzy preference-based conflict situation, and a pair of thresholds $(\beta_*,\beta^*)$ such that $0\leqslant\beta_*<\beta^*\leqslant1$. Then the strong conflict agents $\mathbb{SA}(J)$, the weak conflict agents $\mathbb{WA}(J)$, and the non-conflict agents $\mathbb{NA}(J)$ towards a bundle of issues $J\subseteq I$ are defined by:	
\begin{enumerate}[label={(\arabic*)}]
\item $\mathbb{SA}(J)=\{a\in A\mid\ CM(a,J)\geqslant \beta^*\}$;
\item $\mathbb{WA}(J)=\{a\in A\mid\ \beta_*<CM(a,J)<\beta^*\}$;
\item$\mathbb{NA}(J)=\{a\in A\mid\ CM(a,J)\leqslant \beta_*\}$.
\end{enumerate}	
\end{definition}

The trisection $\langle\langle \mathbb{SA}(J),\mathbb{WA}(J), \mathbb{NA}(J)\rangle\rangle$ partitions the agent set into three disjoint subsets corresponding to high, medium, and low conflict degrees towards a bundle of issues $J\subseteq I$.

\begin{definition}\label{D17}
Let $IFPS=(A,I,\{\widetilde{R}_a \mid a \in A\})$ be an intuitionistic fuzzy preference-based conflict situation, the conflict measure $CM: 2^A\times I\longrightarrow [0,1]$ is defined by: for multiple agents $B\subseteq A$ and an issue $i\in I$,
\begin{align}
CM(B,i)=	\frac{1}{(|I|-1)\cdot |B|\cdot(|B|-1)} \sum_{\substack{j \in I\setminus \{i\}}} \sum_{\substack{a,b \in B}} CF_{ij}(a,b). 
\end{align}
\end{definition}
		
The conflict measure characterizes the conflict degree among multiple agents with respect to a specific issue. Specifically, it captures the conflict degree on that issue relative to all other issues.

\begin{definition}\label{D18}
Let $IFPS=(A,I,\{\widetilde{R}_a \mid a \in A\})$ be an intuitionistic fuzzy preference-based conflict situation, and a pair of thresholds $(\gamma_*,\gamma^*)$ such that $0\leqslant\gamma_*<\gamma^*\leqslant1$.Then the strong conflict issues $\mathbb{SI}(B)$, weak conflict issues $\mathbb{WI}(B)$, and non-conflict issues $\mathbb{NI}(B)$ for multiple agents $B\subseteq A$ are defined by:	
\begin{enumerate}[label={(\arabic*)}]
\item $\mathbb{SI}(B)=\{i\in I\mid\ CM(B,i)\geqslant \gamma^*\}$;
\item $\mathbb{WI}(B)=\{i\in I\mid\ \gamma_*<CM(B,i)<\gamma^*\}$;
\item $\mathbb{NI}(B)=\{i\in I\mid\ CM(B,i)\leqslant\gamma_*\}$.
\end{enumerate}	
\end{definition}

The trisection $\langle\langle \mathbb{SI}(B),\mathbb{WI}(B), \mathbb{NI}(B) \rangle\rangle$ partitions the issue set into three mutually exclusive subsets corresponding to high, moderate, and low levels of conflicts towards multiple agents $B\subseteq A$, respectively.

\begin{example}\label{example6}
(Continued from Example \ref{example3}). We illustrate that how to calculate the trisection of the agent set towards the whole issues and the trisection of the issue set towards the whole agents as follows.

$(1)$ According to Definition \ref{D19}, the conflict degree of six agents with respect to the issue set $I$ are computed as follows:
\begin{align}
&CM(a_1,I)=0.63,\ CM(a_2,I)= 0.54,\ CM(a_3,I)=0.57, \notag \\
&CM(a_4,I)=0.49,\ CM(a_5,I)= 0.47,\ CM(a_6,I)=0.50.
\end{align}
By taking the thresholds $\beta^*=0.6$ and $\beta_*=0.5$, the trisection of the agent set $A$ towards the whole set $I$ is derived as follows: \begin{align}
\mathbb{SA}(I)=\{a_1\},\ \mathbb{WA}(I)=\{a_2,a_3\},\ \mathbb{NA}(I)=\{a_4, a_5, a_6\}.
\end{align}
			
$(2)$ According to Definition \ref{D17}, the conflict degrees among all agents towards five issues are calculated as follows:
\begin{align} 
&CM(A,i_1)= 0.50,\ CM(A,i_2)= 0.53,\ CM(A,i_3) = 0.52, \notag \\
&CM(A,i_4)= 0.53,\ CM(A,i_5)= 0.59.
\end{align}
By taking the thresholds $\gamma^*=0.55$ and $\gamma_*=0.52$, the trisection of the issue set $I$ towards the whole set $A$ is obtained as follows: 
\begin{align}
\mathbb{SI}(A)=\{i_5\},\ \mathbb{WI}(A)=\{i_2,i_4\},\  \mathbb{NI}(A)=\{i_1,i_3\}.
\end{align}
\end{example}

\subsection{Calculation of Thresholds with Bayesian Minimum Risk Theory}\label{Thedeterminationof thresholds}

For two agents \(a, b \in A\), the symbols \(T_S\), \(T_W\), and \(T_N\) denote the actions that assign agent \(b\) to the conflict agents \(\mathbb{IR}^{\asymp}_{J}(a)\), the neutrality agents \(\mathbb{IR}^{\approx}_{J}(a)\), and the alliance agents \(\mathbb{IR}^{=}_{J}(a)\) of agent $a$, respectively. The relative loss functions associated with these three actions under different situations are summarized in Table~\ref{table:1}. Specifically, when \(b \in \mathbb{IR}^{\asymp}_{J}(a)\), the losses incurred by actions \(T_S\), \(T_W\), and \(T_N\) are denoted by \(\lambda_{SS}\), \(\lambda_{WS}\), and \(\lambda_{NS}\), respectively. Likewise, when \(b \in \mathbb{IR}^{=}_{J}(a)\), the corresponding losses of these actions are given by \(\lambda_{SN}\), \(\lambda_{WN}\), and \(\lambda_{NN}\), respectively.
In fact, an increase of the conflict degree between two agents corresponds to a decrease of the loss incurred by assigning them to the strong conflict agents. 
For simplicity, the relative loss functions based on the conflict function are given by:
\begin{align}\label{relative loss}
\lambda_{SN}^*&=\frac{1}{|A|(|A|-1)}\sum_{\substack{a,b \in A }}(1-CF_J(a,b)),\\	\lambda_{NS}^*&=\frac{1}{|A|(|A|-1)}\sum_{\substack{a,b \in A }}CF_J(a,b).
\end{align}
		 
For an agent $a\in A$ and a parameter $\sigma\in(0,1)$, the relative loss functions are defined as follows:
\begin{enumerate}[label={(\arabic*)}]
\item For $b \in \mathbb{IR}^{\asymp}_{J}(a)$, $\lambda_{SS}=0$, $\lambda_{WS}=\sigma\lambda^*_{NS}$, and $\lambda_{NS}=\lambda^*_{NS}$;
\item For $b \in \mathbb{IR}^{=}(a)$, $\lambda_{NN}=0$, $\lambda_{WN}=\sigma\lambda^*_{SN}$, and $\lambda_{SN}=\lambda^*_{SN}$.
\end{enumerate}	
\begin{table}
\centering
\caption{The relative loss functions for three actions \(T_S\), \(T_W\), and \(T_N\).}
\label{table:1}
\setlength{\tabcolsep}{40pt} 
\renewcommand{\arraystretch}{1.2}
\begin{tabularx}{\textwidth}{lll}
\toprule
$\mathrm{Action}$ & $b \in \mathbb{IR}^{\asymp}_{J}(a)$ & $ b \in \mathbb{IR}^{=}_{J}(a) $ \\
\midrule
$T_S$ & $\lambda_{SS}=0$ & $\lambda_{SN}=\lambda^*_{SN}$ \\
$T_W$ & $\lambda_{WS}=\sigma\lambda^*_{NS}$ & $\lambda_{WN}=\sigma\lambda^*_{SN}$  \\
$T_N$ & $\lambda_{NS}=\lambda^*_{NS}$ & $\lambda_{NN}=0$  \\
\bottomrule
\end{tabularx}
\end{table}		
Then, the expected losses $R^a(T_S|b), R^a(T_W|b), R^a(T_N|b)$ associated with taking the individual action for agent $b$ can be expressed as:
\begin{align}
R^a(T_S|b)&=\lambda_{SN}(1-CF_J(a,b));\\
R^a(T_N|b)&=\lambda_{NS}CF_J(a,b);\\
R^a(T_W|b)&=\sigma\lambda_{NS}CF_J(a,b)+\sigma\lambda_{SN}(1-CF_J(a,b)).
\end{align}
The Bayesian decision-theoretic rough sets yields the following minimum expected-cost decision rules as follows:\\
($r_S$): If $R^a(T_S|b)\leqslant R^a(T_W|b)$ and $R^a(T_S|b)\leqslant R^a(T_N|b))$, then $b\in \mathbb{IR}^{\asymp}_{J}(a)$;\\
($r_W$): If $R^a(T_W|b)\leqslant R^a(T_S|b)$ and $R^a(T_W|b)\leqslant R^a(T_N|b))$, then $b\in \mathbb{IR}^{\approx}_{J}(a)$;\\
($r_N$): If $R^a(T_N|b))\leqslant R^a(T_S|b)$ and $R^a(T_N|b))\leqslant R^a(T_W|b)$, then $b\in \mathbb{IR}^{=}_{J}(a)$.

Assume that $\lambda_{SS}\leq\lambda_{WS}\leq\lambda_{NS}$, $\lambda_{NN}\leq\lambda_{WN}\leq\lambda_{SN}$, the rules ($r_S$), ($r_W$), and ($r_N$) can be simply expressed by:\\
($r_S$): $\mathbb{IR}^{\asymp}_{J}(a)=\{b\in A\mid CF_{J}(a,b)\geqslant \alpha^*\}$;\\
($r_W$): $\mathbb{IR}^{\approx}_{J}(a)=\{b\in A\mid \alpha_*<CF_{J}(a,b)<\alpha^*\}$;\\
($r_N$): $\mathbb{IR}^{=}_{J}(a)=\{b\in A\mid CF_{J}(a,b)\leqslant \alpha_*\}$, where	
\begin{align}\label{threshold}
\alpha^*&=\frac{\lambda_{SN}-\sigma\lambda_{SN}}{\lambda_{SN}-\sigma\lambda_{SN}+\sigma\lambda_{NS}},\ \alpha_*=\frac{\sigma\lambda_{SN}}{\sigma\lambda_{SN}+\lambda_{NS}-\sigma\lambda_{NS}}.
\end{align}
It is evident that the thresholds \(\alpha_*\) and \(\alpha^*\) satisfy \(0 \leqslant \alpha_* < \alpha^* \leqslant 1\), with the parameter \(\sigma\) constrained to the interval \(0 < \sigma < 0.5\).

\begin{example}\label{example4}
(Continued from Example \ref{example3}). We illustrate that how to calculate the trisections of the set of agent pairs, the agent set and the issue set with the relative loss functions as follows.
First, for the bundle of issues $J=I$, the relative losses are obtained according to Equation~(\ref{relative loss}) as follows:
\begin{align}
\lambda_{SN}&=\frac{1}{|A|(|A|-1)}\sum_{\substack{a,b \in A }}(1-CF_I(a,b))=0.47,\\
\lambda_{NS}&=\frac{1}{|A|(|A|-1)}\sum_{\substack{a,b \in A }}CF_I(a,b)=0.53.
\end{align}
Second, by taking the parameter $\sigma= 0.44$, according to Equation (\ref{threshold}), we have $\alpha^*=0.53$ and $\alpha_*=0.41$, and show the trisection of the set of agent pairs $A \times A$ towards a single agent by Table~\ref{tab:trisection}. Third, by taking $\beta^*=0.53$ and $\beta_*=0.41$, we have the trisection of the set of agents $A$ towards the whole set $I$ as follows: \begin{align}
\mathbb{SA}(I)=\{a_1,a_2,a_3\},\ \mathbb{WA}(I)=\{a_4, a_5, a_6\},\ \mathbb{NA}(I)=\emptyset.
\end{align}
Finally, by taking $\gamma^*=0.53$ and $\gamma_*=0.41$, we calculate the trisection of the set of issues $I$ towards the whole set $A$ as follows: 
\begin{align}
\mathbb{SI}(A)=\{i_2,i_4,i_5\},\ \mathbb{WI}(A)=\{i_1,i_3\},\ \mathbb{NI}(A)=\emptyset.
\end{align}
\end{example}

\begin{table}[htp]
\centering
\caption{The trisection of agents for an individual agent over $I$ ($\alpha^*=0.53$, $\alpha_*=0.41$).}
\label{tab:trisection}
\setlength{\tabcolsep}{30pt} 
\renewcommand{\arraystretch}{1.2}
\begin{tabularx}{\textwidth}{llll}
\toprule
$A$ & $\mathbb{IR}^{\asymp}_{J}(a)$ & $\mathbb{IR}^{\approx}_{J}(a)$ & $\mathbb{IR}^{=}_{J}(a)$ \\
\midrule
$a_1$ & $\{a_2, a_3,a_5, a_6\}$ & $\emptyset $ & $\{a_1, a_4\} $ \\
$a_2$ & $\{a_1\}$ & $\{a_3,a_4, a_5\}$ & $\{a_2,a_6\} $ \\
$a_3$ & $\{a_1, a_4,a_6\} $ & $\{a_2,a_5\} $ & $\{a_3\} $ \\
$a_4$ & $\{a_3\}$ & $\{a_2,a_6\}$ & $\{a_1,a_4,a_5\}$ \\
$a_5$ & $\{a_1\}$ & $\{ a_2,a_3,a_6\}$ & $\{a_4,a_5\}$ \\
$a_6$ & $\{a_1,a_3\}$ & $\{ a_4,a_5\}$ & $\{a_2, a_6\}$ \\
\bottomrule
\end{tabularx}
\end{table}

\section{Construction of Feasible Strategies via an Adjustment Mechanism}\label{sec:conflictresolution}
		
In this section, we present feasible strategies based on adjustment mechanisms within intuitionistic fuzzy preference-based conflict situations.

\begin{definition}\label{D}
Let $IFPS=(A,I,\{\widetilde{R}_a \mid a \in A\})$ be an intuitionistic fuzzy preference-based conflict situation, the conflict measure $CM: 2^A\times 2^I\longrightarrow [0,1]$ for multiple agents $B\subseteq A$ towards a bundle of issues $J\subseteq I$ is defined by:
\begin{align}
CM(B,J) = 
\frac{1}{|B|}\sum_{\substack{a\in B}}CM(a,J). 
\end{align}
\end{definition}	

The conflict measure given by Definition~\ref{D} quantifies the conflict degree among multiple agents towards an issue bundle within intuitionistic fuzzy preference-based conflict situations. Moreover, a higher conflict degree indicates a strong disagreement among multiple agents regarding the issue bundle, which complicates the achievement of consensus and necessitates the implementation of feasible strategies.

Follows, we demonstrate that how to design feasible strategies by adjusting agents' preferences so that the conflict degree satisfies the condition such as \( CM(A,I) \leq \kappa \), where \(\kappa \) is a critical threshold. 
In practice, the conflict degree of intuitionistic fuzzy preference-based conflict situations can be reduced by appropriately modifying agents' preferences. Since the conflict degree \( CM(a,I) \) generally varies across agents, a higher value indicates a greater potential for adjustment. The conflict resolution process prioritizes the adjustment of the intuitionistic fuzzy preference of an agent associated with the highest conflict degree.

The optimization function of constructing feasible strategies aims to simultaneously minimize the conflict degree among multiple agents and the magnitude of preference adjustments with a constraint on the number of adjusted intuitionistic fuzzy preferences. That is, it seeks to reduce the conflict degree of the conflict situation while preserving the stability of agents' preferences. Accordingly, the precedence is given to the intuitionistic fuzzy preferences \( \widetilde{R}_{a^*} \), where
$a^* = \arg \max \{ CM(a,I) \mid a \in A \}.$

\begin{definition}\label{adjustment}
Let $IFPS=(A,I,\{\widetilde{R}_a \mid a \in A\})$ be an intuitionistic fuzzy preference-based conflict situation. 
Then the optimization function of constructing feasible strategies for the target agent $a^*$ is defined as:
\begin{align*}
& \operatorname*{min}L= \overline{CM}(a^*,I)+ \rho
\sum_{i\neq j}\Big[|\overline{\mu}_{a^*}(i,j)-\mu_{a^*}(i,j)|+|\overline{\nu}_{a^*}(i,j)-\nu_{a^*}(i,j)|\Big], \\
& \text{s.t.} 
\left\{
\begin{aligned}
& \sum_{i\neq j} (z_{ij}+ z_{ji})=2k, z_{ij} = z_{ji}\in \{0,1\},\\
& |\overline{\mu}_{a^*}(i,j)-\mu_{a^*}(i,j)| \le z_{ij},  |\overline{\nu}_{a^*}(i,j)-\nu_{a^*}(i,j)| \le z_{ij}, \\
& \overline{\mu}_{a^*}(i,j) = \overline{\nu}_{a^*}(j,i), 
\overline{\mu}_{a^*}(i,i) = \overline{\nu}_{a^*}(i,i) = 0.5, \\
& 0 \le \overline{\mu}_{a^*}(i,j) + \overline{\nu}_{a^*}(i,j) \le 1, 
0 \le \overline{\mu}_{a^*}(i,j), \overline{\nu}_{a^*}(i,j) \le 1,
\end{aligned}
\right.
\end{align*}
where $\rho = \frac{1}{|I| (|I| - 1)(|A| - 1)}$ is a balance parameter, 
$k$ denotes the number of adjusted preference pairs,  $k \in \mathbb{Z}$ and $1\leqslant k \leqslant \tfrac{|I|^2 - |I|}{2}$, 
$z_{ij}=1$ indicates that the  intuitionistic fuzzy preference towards the issue pair $(i,j)$ is adjusted, whereas $z_{ij}=0$ indicates that it remains unchanged, 
$\overline{\mu}_{a^*}(i,j)$ and $\overline{\nu}_{a^*}(i,j)$ represent the adjusted intuitionistic fuzzy preference degrees of agent $a^*$ towards the issue pair $(i,j)$, and $\overline{CM}(a^*, I)$ denotes the conflict degree between agent $a^*$ and all other agents over the entire issue set after adjustment.
\end{definition}

The function given by Definition~\ref{adjustment} enforces that exactly 
intuitionistic fuzzy preference degrees towards $k$ issue pairs are adjusted, thereby ensuring the effectiveness of the preference modification. Moreover, the optimization function simultaneously minimizes both the conflict degree and the magnitude of preference adjustments for agent 
$a^*$. The parameter 
$\rho$ guarantees that the adjustment magnitude and the conflict degree are of comparable scale, thereby ensuring a balanced contribution of the two terms to the overall objective.

Subsequently, we design an algorithm of constructing feasible strategies for conflict resolution.

\begin{algorithm}[H]
	\caption{Feasible strategies based on an adjustment mechanism.}
	\label{A3}
	\begin{algorithmic}[1]
		\Require $IFPS = (A, I, \{\widetilde{R}_a \mid a \in A\})$, $\kappa$, $k$, $S$;
		\Ensure $IFPS^{(t)}$;
		\State $t \gets 0$;
		\State $IFPS^{(t)} \gets IFPS$;
		\State Compute $CM^{(t)}(a,I)$ for each agent $a \in A$ by Definition~\ref{D19};
		\State Calculate $CM^{(t)}(A,I)$ by Definition~\ref{D};
		\While{$CM^{(t)}(A,I)>\kappa$}
		\State $a^* \gets \arg\operatorname*{max} \{CM^{(t)}(a,I)\mid a \in A\}$;
		\State Run the Simulated Annealing algorithm $S$ times and select the optimal result to adjust $\widetilde{R}^{(t)}_{a^*}$ by Definition~\ref{adjustment};
		\State Obtain the updated $IFPS^{(t+1)}$;
		\State Compute $CM^{(t+1)}(a,I)$ for each agent $a \in A$;
		\State Calculate $CM^{(t+1)}(A,I)$;
		\State $t \gets t + 1$;
		\EndWhile
		\State \Return $IFPS^{(t)}$.
	\end{algorithmic}
\end{algorithm}

Algorithm~\ref{A3} provides the procedure of constructing feasible strategies for conflict resolution within intuitionistic fuzzy preference-based conflict situations, where $t$ denotes the $t$-th iteration, and $S$ is the total number of runs.
After executing this algorithm $S$ times, the solution that satisfies all constraints and yields the minimum objective value is selected as the  feasible strategy.
Moreover, the time complexity of Steps 3 and 4 is \(O(|A|^2 \cdot |I|^2) \).
If $N$ is the number of iterations corresponding to Steps 5-12, and $M$ is the number of iterations in each Simulated Annealing process, the time complexity of each iteration is \(O(S\cdot M\cdot|A|^2 \cdot |I|^2+|A|^2 \cdot |I|^2)\), then the total time complexity is \(O(N\cdot S\cdot M\cdot|A|^2 \cdot |I|^2) \).

\begin{table}[htp]
\centering
\caption{The feasible strategy for the Middle East Conflict ($t=1$).}
\label{table strategy}
\setlength{\tabcolsep}{8pt}
\renewcommand{\arraystretch}{1.2}
\begin{tabular}{cccccc} 
\toprule
{No.} &$\widetilde{R}_{ a_1}(i,j)$ & $\widetilde{R}_{ a_1}(i,j) \rightarrow \overline{\widetilde{R}}_{ a_1}(i,j)$ & $\Delta \widetilde{R}_{a_1}(i,j)$ & $L^{(1)}$ & $CM^{(1)}$ \\
\midrule
\multirow{2}{*}{$1$} &
$\widetilde{R}_{ a_1}(i_1,i_5)$ & $(0.97,0.02) \rightarrow (0.40,0.48)$ & $(-0.57,0.46)$ & \multirow{2}{*}{$0.5984$} & \multirow{2}{*}{$0.5112$}\\
& $\widetilde{R}_{ a_1}(i_3,i_5)$ & $(0.90,0.07) \rightarrow (0.51,0.38)$ & $(-0.39,0.31)$ & & \\
\cmidrule(lr){1-6}
\multirow{2}{*}{$2$} &
$\widetilde{R}_{a_1}(i_3,i_4)$ & $(0.96,0.03) \rightarrow (0.47, 0.43)$ & $(-0.49, 0.40)$ & \multirow{2}{*}{$0.5969$} & \multirow{2}{*}{$0.5113$} \\
& $\widetilde{R}_{ a_1}(i_3,i_5)$ & $(0.90,0.07) \rightarrow (0.47,0.40)$ & $(-0.43,0.33)$ & & \\
\cmidrule(lr){1-6}
\multirow{2}{*}{$3$} &
$\widetilde{R}_{ a_1}(i_3,i_4)$ & $(0.96,0.03) \rightarrow (0.52, 0.38)$ & $(-0.44, 0.35)$ & \multirow{2}{*}{$0.5957$} & \multirow{2}{*}{$0.5108$} \\
& $\widetilde{R}_{a_1}(i_3,i_5)$ & $(0.90,0.07) \rightarrow (0.42,0.45)$ & $(-0.48,0.38)$ & & \\
\cmidrule(lr){1-6}
\multirow{2}{*}{$4$} &
$\widetilde{R}_{ a_1}(i_3,i_4)$ & $(0.96,0.03) \rightarrow (0.43, 0.45)$ & $(-0.53, 0.42)$ & \multirow{2}{*}{\textcolor{blue}{{$0.5943$}}} & \multirow{2}{*}{$0.5090$} \\
& $\widetilde{R}_{ a_1}(i_3,i_5)$ & $(0.90,0.07) \rightarrow (0.41,0.48)$ & $(-0.49,0.41)$ & & \\
\cmidrule(lr){1-6}
\multirow{2}{*}{$5$} &
$\widetilde{R}_{ a_1}(i_3,i_4)$ & $(0.96,0.03) \rightarrow (0.45, 0.44)$ & $(-0.51, 0.41)$ & \multirow{2}{*}{$0.5949$} & \multirow{2}{*}{$0.5098$} \\
& $\widetilde{R}_{ a_1}(i_3,i_5)$ & $(0.90,0.07) \rightarrow (0.40,0.42)$ & $(-0.50,0.35)$ & & \\
\cmidrule(lr){1-6}
\multirow{2}{*}{$6$} &
$\widetilde{R}_{a_1}(i_3,i_4)$ & $(0.96,0.03) \rightarrow (0.48, 0.41)$ & $(-0.48, 0.38)$ & \multirow{2}{*}{$0.5948$} & \multirow{2}{*}{$0.5098$} \\
& $\widetilde{R}_{a_1}(i_3,i_5)$ & $(0.90,0.07) \rightarrow (0.41,0.49)$ & $(-0.49,0.42)$ & & \\
\cmidrule(lr){1-6}
\multirow{2}{*}{$7$} &
$\widetilde{R}_{a_1}(i_3,i_4)$ & $(0.96,0.03) \rightarrow (0.51, 0.37)$ & $(-0.45, 0.34)$ & \multirow{2}{*}{$0.5959$} & \multirow{2}{*}{$0.5107$} \\
& $\widetilde{R}_{a_1}(i_3,i_5)$ & $(0.90,0.07) \rightarrow (0.41,0.47)$ & $(-0.49,0.40)$ & & \\
\cmidrule(lr){1-6}
\multirow{2}{*}{$8$} &
$\widetilde{R}_{a_1}(i_3,i_4)$ & $(0.96,0.03) \rightarrow (0.47, 0.42)$ & $(-0.49, 0.39)$ & \multirow{2}{*}{$0.5950$} & \multirow{2}{*}{$0.5100$} \\
& $\widetilde{R}_{a_1}(i_3,i_5)$ & $(0.90,0.07) \rightarrow (0.42,0.45)$ & $(-0.48,0.38)$ & & \\
\cmidrule(lr){1-6}
\multirow{2}{*}{$9$} &
$\widetilde{R}_{a_1}(i_3,i_4)$ & $(0.96,0.03) \rightarrow (0.45, 0.42)$ & $(-0.51, 0.39)$ & \multirow{2}{*}{$0.5969$} & \multirow{2}{*}{$0.5108$} \\
& $\widetilde{R}_{a_1}(i_3,i_5)$ & $(0.90,0.07) \rightarrow (0.41,0.39)$ & $(-0.49,0.32)$ & & \\
\cmidrule(lr){1-6}
\multirow{2}{*}{$10$} &
$\widetilde{R}_{a_1}(i_1,i_5)$ & $(0.97,0.02) \rightarrow (0.40, 0.48)$ & $(-0.57, 0.46)$ & \multirow{2}{*}{$0.5957$} & \multirow{2}{*}{$0.5094$} \\
& $\widetilde{R}_{a_1}(i_3,i_5)$ & $(0.90,0.07) \rightarrow (0.42,0.44)$ & $(-0.48,0.37)$ & & \\
\bottomrule 
\end{tabular}
\end{table}

\begin{example}\label{Feasiblestrategyconflictresolution}
(Continued from Example~\ref{example1}).
We utilize Pawlak's Middle East Conflict to illustrate that how to construct feasible strategies for conflict resolution as follows.
According to Definition~\ref{D}, the conflict degree of the intuitionistic fuzzy preference-based conflict situation in Example~\ref{example1} is computed by:
\begin{align*}
CM(A,I)=\frac{1}{6}\times(0.63+0.54+0.57+0.49+0.47+0.50)= 0.53.
\end{align*}
Then, by taking the parameters \(\kappa = 0.47\), \(k = 2\), and $S=10$, the procedure of conflict resolution proceeds through three iterations, with the simulated annealing algorithm executed ten times in each iteration to minimize the objective function. The corresponding results are summarized as follows.

\begin{table}[htp]
\centering
\caption{The feasible strategy for the Middle East Conflict.}
\label{table:strategy}
\setlength{\tabcolsep}{3pt}
\renewcommand{\arraystretch}{1.2}
\begin{tabular}{ccccccccc} 
\toprule
$t$ & $a^*$ &$\widetilde{R}_{ a^*}(i,j)$ & $\widetilde{R}_{ a^*}(i,j) \rightarrow \overline{\widetilde{R}}_{ a^*}(i,j)$ & $\Delta \widetilde{R}_{ a^*}(i,j)$ & $L^{(t)}$ & $CM^{(t)}$  & $(\alpha^*,\alpha_*)^{(t)}$ \\
\midrule
\multirow{2}{*}{1} & \multirow{2}{*}{$a_1$} &
$\widetilde{R}_{ a_1}(i_3,i_4)$ & $(0.96,0.03) \rightarrow (0.43, 0.45)$ & $(-0.53, 0.42)$ & \multirow{2}{*}{0.5943} & \multirow{2}{*}{0.5090} & \multirow{2}{*}{(0.43,0.55)} \\
& & $\widetilde{R}_{a_1}(i_3,i_5)$ & $(0.90,0.07) \rightarrow (0.41,0.48)$ & $(-0.49,0.41)$ & & & \\
\cmidrule(lr){1-8}
\multirow{2}{*}{2} & \multirow{2}{*}{$a_3$} &
$\widetilde{R}_{a_3}(i_3,i_5)$ & $(0.08,0.17) \rightarrow (0.41, 0.44)$ & $(0.33, 0.27)$ & \multirow{2}{*}{0.5073} & \multirow{2}{*}{0.4829} & \multirow{2}{*}{(0.45,0.57)} \\
& & $\widetilde{R}_{ a_3}(i_4,i_5)$ & $(0.15,0.05) \rightarrow (0.41,0.47)$ & $(0.26,0.42)$ & & & \\
\cmidrule(lr){1-8}
\multirow{2}{*}{3} & \multirow{2}{*}{$a_1$} &
$\widetilde{R}_{a_1}(i_1,i_5)$ & $(0.97,0.02) \rightarrow (0.43, 0.44)$ & $(-0.54, 0.42)$ & \multirow{2}{*}{0.5136} & \multirow{2}{*}{0.4612} & \multirow{2}{*}{(0.47,0.59)} \\
& & $\widetilde{R}_{ a_1}(i_4,i_5)$ & $(0.95,0.01) \rightarrow (0.43,0.46)$ & $(-0.52,0.45)$ & & & \\
\bottomrule 
\end{tabular}
\end{table}

\begin{enumerate}[label={(\arabic*)}]
\item  Table \ref{table strategy} illustrates the outcomes of each Simulated Annealing algorithm run in the first iteration. Among these runs, the fourth achieves the minimum objective value $L^{(1)}=0.5943$, and is therefore selected as the final result for this iteration. In Table~\ref{table:strategy},  $\widetilde{R}_{ a^*}(i,j)$ denotes the preference element requiring adjustment, $\widetilde{R}_{ a^*}(i,j) \rightarrow \overline{\widetilde{R}}_{ a^*}(i,j)$ represents the corresponding feasible strategy, and $\Delta \widetilde{R}_{ a^*}(i,j)$ indicates the adjustment magnitude. Furthermore, 
Table~\ref{table:strategy} summarizes the results obtained by applying  Algorithm~\ref{A3} to the Middle East Conflict. As the number of iterations increases, the conflict degree  decreases progressively, indicating that the proposed adjustment mechanism is effective for mitigating conflicts.
\item The adjusted intuitionistic fuzzy preferences are abbreviated below, with the updated values highlighted in blue. 
{\setlength{\arraycolsep}{1.9pt}
\[
\begin{array}{ccc}
\widetilde{R}_{a_1}^{\mu} =
\begin{bmatrix}
0.50 & 0.90 & 0.93 & 0.89 & \textcolor{blue}{0.43} \\
0.03 & 0.50 & 0.95 & 0.90 & 0.93 \\
0.05 & 0.04 & 0.50 &\textcolor{blue}{0.43} & \textcolor{blue}{0.41} \\
0.06 & 0.08 & \textcolor{blue}{0.45} & 0.50 & \textcolor{blue}{0.43} \\
\textcolor{blue}{0.44} & 0.05 & \textcolor{blue}{0.48} & \textcolor{blue}{0.46 } & 0.50
\end{bmatrix},
&
\widetilde{R}_{a_3}^{\mu} = 
\begin{bmatrix}
0.50 & 0.10 & 0.88 & 0.12 & 0.85 \\
0.89 & 0.50 & 0.10 & 0.09 & 0.13 \\
0.10 & 0.89 & 0.50 & 0.10 & \textcolor{blue}{0.41} \\
0.86 & 0.90 & 0.87 & 0.50 & \textcolor{blue}{0.41} \\
0.10 & 0.07 & \textcolor{blue}{0.44} & \textcolor{blue}{0.47} & 0.50
\end{bmatrix}.
\end{array}
\]
}
Moreover, Table~\ref{tab:trisection changes} compares the trisection of agent pairs before and after the implementation of the feasible strategy. The results indicate that the number of conflict agents either decreases or remains unchanged, while most neutrality agents are reduced and all alliance agents increase. These findings provide empirical evidence of the effectiveness of the proposed feasible strategy in alleviating conflicts.

\begin{table}[htp]
\centering
\caption{Changes in trisections of agent pairs ($\alpha_*=0.47,\ \alpha^*=0.59$).}
\label{tab:trisection changes}
\setlength{\tabcolsep}{6pt} 
\renewcommand{\arraystretch}{1.2}
\begin{tabularx}{\textwidth}{llll}
\toprule
$A$ & $\mathbb{IR}^{\asymp}_{I}(a)$ & $\mathbb{IR}^{\approx}_{I}(a)$& $\mathbb{IR}^{=}_{I}(a)$ \\
\midrule
$a_1$ & $\{a_2, a_3,a_5, a_6\}\rightarrow \textcolor{blue}{\{a_2\}}$ & $\emptyset\rightarrow \textcolor{blue}{\{a_3\}} $ & $\{a_1, a_4\}\rightarrow \textcolor{blue}{\{a_1,a_4,a_5,a_6\}} $ \\
$a_2$ & $\{a_1\}$ & $\{a_3,a_4, a_5\}\rightarrow \textcolor{blue}{\{a_4\}} $ & $\{a_2,a_6\}\rightarrow \textcolor{blue}{\{a_2,a_3,a_5,a_6\}} $ \\
$a_3$ & $\{a_1, a_4,a_6\}\rightarrow \textcolor{blue}{\{a_4\}} $ & $\{a_2,a_5\}\rightarrow \textcolor{blue}{\{a_1,a_6\}}  $ & $\{a_3\}\rightarrow \textcolor{blue}{\{a_2,a_3,a_5\}} $ \\
$a_4$ & $\{a_3\}$ & $\{a_2,a_6\}\rightarrow \textcolor{blue}{\{a_2\}}$ & $\{a_1,a_4,a_5\}\rightarrow \textcolor{blue}{\{a_1,a_4,a_5,a_6\}}$ \\
$a_5$ & $\{a_1\}\rightarrow \textcolor{blue}{\emptyset}$ & $\{ a_2,a_3,a_6\}\rightarrow \textcolor{blue}{\emptyset}$ & $\{a_4,a_5\}\rightarrow \textcolor{blue}{\{a_1,a_2,a_3,a_4,a_5,a_6\}}$ \\
$a_6$ & $\{a_1,a_3\}\rightarrow \textcolor{blue}{\emptyset}$ & $\{ a_4,a_5\}\rightarrow \textcolor{blue}{\{a_3\}}$ & $\{a_2, a_6\}\rightarrow \textcolor{blue}{\{a_1,a_2,a_4,a_5,a_6\}}$ \\
\bottomrule
\end{tabularx}
\end{table}

\item Tables~\ref{tab trisection_changes} and \ref{tab trisectionissue_changes} illustrates the evolution of the trisections of the agent set and the issue set across all iterations with thresholds $(\beta_*,\beta^*)=(0.47,0.59)$, and $(\gamma_*,\gamma^*)=(0.47,0.59)$, respectively. The feasible strategies induce substantial changes in the trisection structures. Within the agent set, agent $a_2$ is downgraded to a weak conflict agent, while agents $a_1$ and $a_3$ shift from strong conflict to non-conflict agents. Similarly, agents $a_4$, $a_5$ and $a_6$ move from weak conflict to non-conflict agents. Similarly, within the issue set, issue $i_2$ is downgraded from a strong to a weak conflict issue, issues $i_1$ and $i_3$ are resolved to non-conflict issues, and issues $i_4$ and $i_5$ also become non-conflict issues. These results demonstrate the effectiveness of the proposed feasible strategy in progressively mitigating conflicts.

\begin{table}[htp]
\centering
\caption{Changes in trisections of the agent set across iterations.}
\label{tab trisection_changes}
\setlength{\tabcolsep}{10pt} 
\renewcommand{\arraystretch}{1.2}
\begin{tabularx}{\textwidth}{lllll}
\toprule 
TS& t=0 & t=1 & t=2 & t=3 \\
\midrule
$\mathbb{SA}(I)$ & $\{a_1,a_2,a_3\}$ & $\{a_1,a_3\}$ & $\emptyset$ & $\emptyset$ \\
$\mathbb{WA}(I)$ & $\{a_4,a_5,a_6\}$ & $\{a_2,a_4,a_5,a_6\}$ & $\{a_1,a_2,a_3,a_4,a_6\}$ & $\{a_2\}$ \\
$\mathbb{NA}(I)$ & $\emptyset$ & $\emptyset$ & $\{a_5\}$ & $\{a_1,a_3,a_4,a_5,a_6\}$ \\
\bottomrule
\end{tabularx}
\end{table}

\begin{table}[htp]
	\centering
	\caption{Changes in trisections of the issue set across iterations.}
	\label{tab trisectionissue_changes}
	\setlength{\tabcolsep}{15pt} 
	\renewcommand{\arraystretch}{1.2}
	\begin{tabularx}{\textwidth}{lllll}
		\toprule 
		TS& t=0 & t=1 & t=2 & t=3 \\
		\midrule
		$\mathbb{SI}(A)$ & $\{i_2,i_4,i_5\}$ & $\{i_5\}$ & $\emptyset$ & $\emptyset$ \\
		$\mathbb{WI}(A)$ & $\{i_1,i_3\}$ & $\{i_1,i_2,i_3,i_4\}$ & $\{i_1,i_2,i_4,i_5\}$ & $\{i_2\}$ \\
		$\mathbb{NI}(A)$ & $\emptyset$ & $\emptyset$ & $\{i_3\}$ & $\{i_1,i_3,i_4,i_5\}$ \\
		\bottomrule
	\end{tabularx}
\end{table}

\item The evolution of trisection cardinalities for both the agent and issue sets is illustrated in Figures~\ref{fig:agent3a} and \ref{fig:agent3b}, respectively. The observed trends reveal a substantial reduction in the strong conflict agents accompanied by a corresponding increase in the non-conflict agents.
Moreover, Figures~\ref{fig:agent} and \ref{fig:issue} depict the evolution of conflict degrees. It is evident that the conflict degrees of both agents and issues decrease following the implementation of the feasible strategy. Notably, agents 
$a_1$ and $a_3$
undergo multiple preference modifications, leading to a more pronounced reduction in their conflict degrees relative to other agents.
These results demonstrate that the proposed feasible strategy is effective in mitigating conflicts and facilitating conflict resolution.

\item The influence of parameter $\sigma$ on the thresholds $\alpha_*$ and $\alpha^*$ is illustrated by Figure~\ref{fig:sensitivity}, in which \(\alpha_*\) increases monotonically with \(\sigma\), whereas \(\alpha^*\) exhibits an opposite trend, decreasing as \(\sigma\) increases. When the parameter \(\sigma\) approaches $0.5$, the thresholds \(\alpha_*\) and \(\alpha^*\) become nearly identical. Concurrently, the increase of the threshold \(\sigma\) leads to a rise of the number of strong and non-conflict agents for each agent, while the number of weak conflict agents correspondingly decreases.

\begin{figure}[htp]
\begin{minipage}{0.48\textwidth}
\centering
\includegraphics[width=\textwidth]{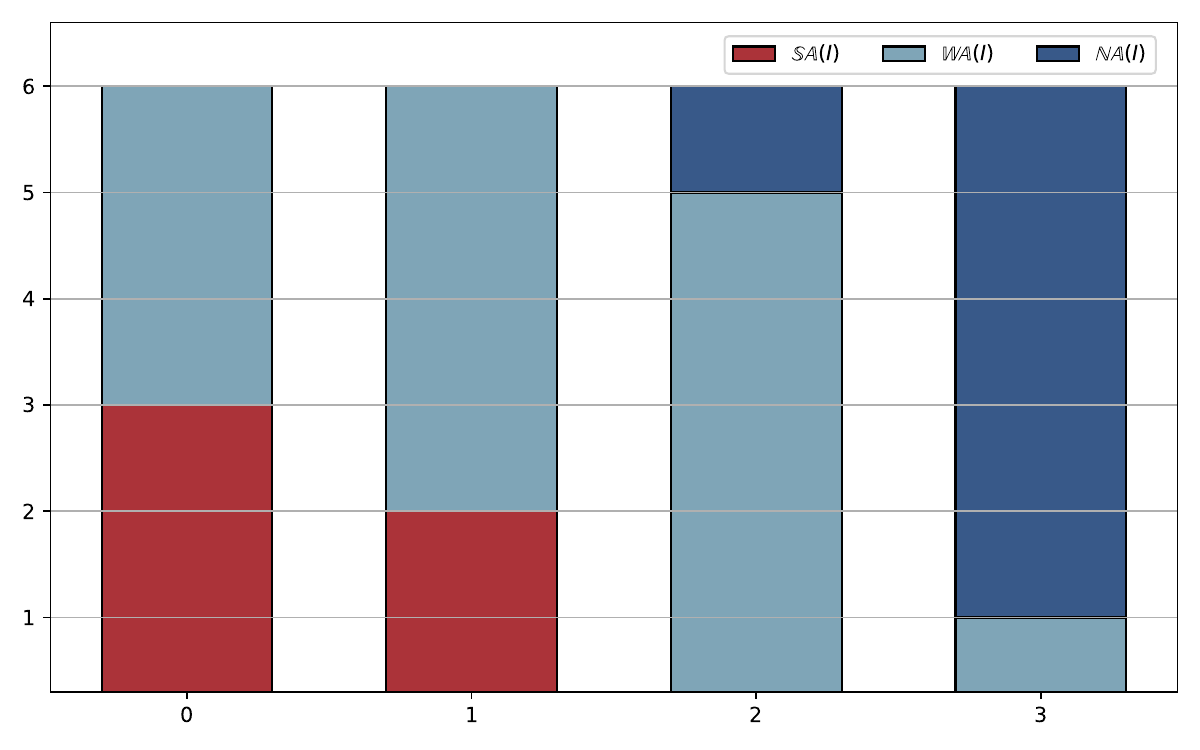}
\captionof{figure}{Changes of strong, weak, and non- conflict agents in each iteration.}
\label{fig:agent3a}
\end{minipage}\hfill
\begin{minipage}{0.48\textwidth}
\centering
\includegraphics[width=\textwidth]{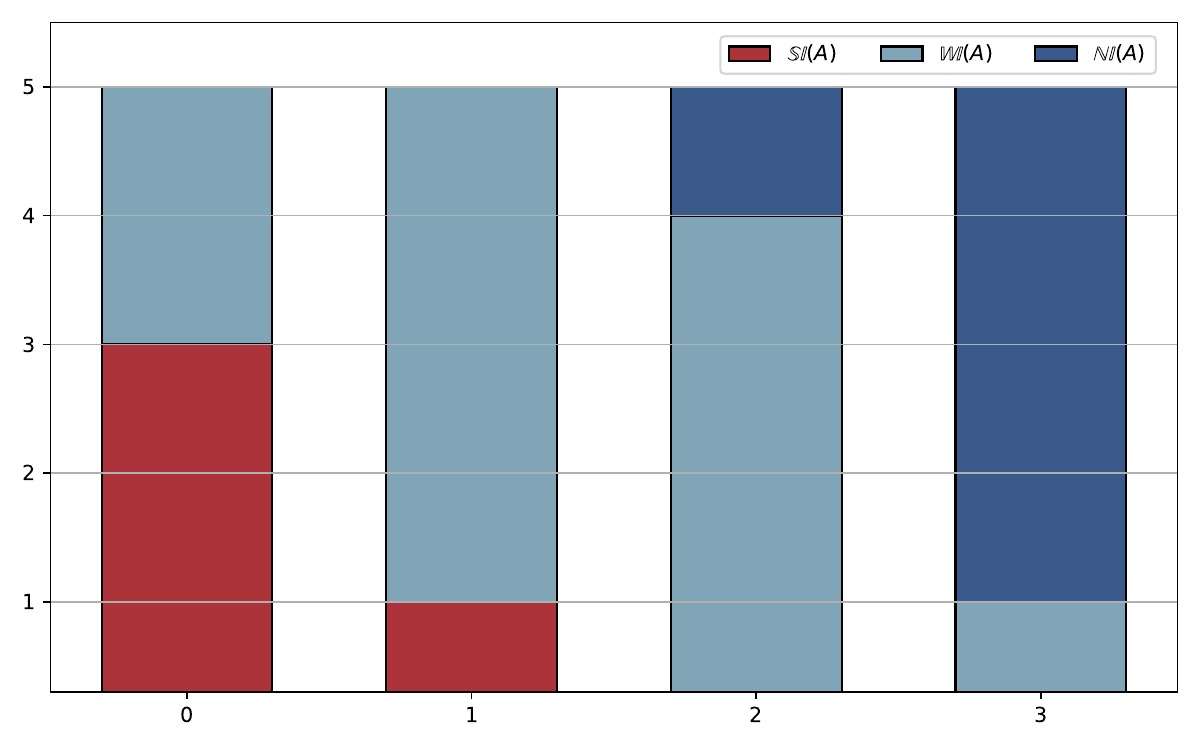}
\captionof{figure}{Changes of strong, weak, and non- conflict issues in each iteration.}
\label{fig:agent3b}
\end{minipage}
\end{figure}

\begin{figure}[htp]
\begin{minipage}{0.48\textwidth}
\centering
\includegraphics[width=\textwidth]{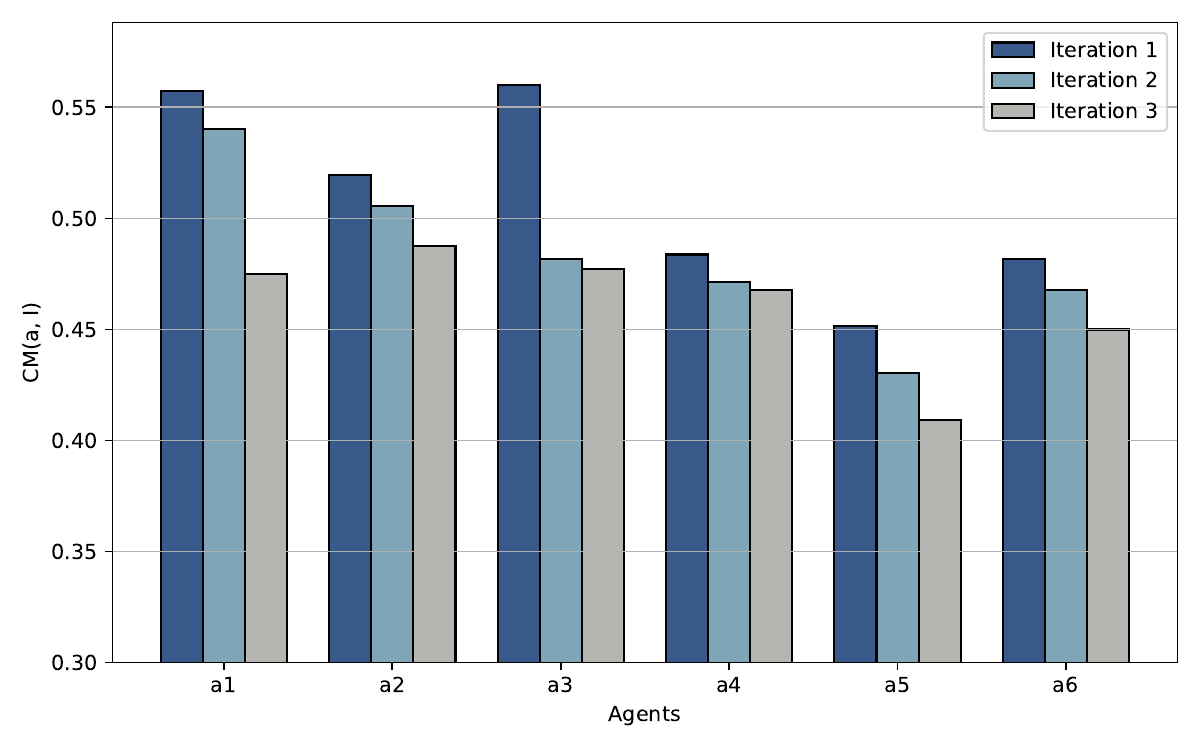}
\captionof{figure}{Changes of the conflict degree of an individual agent.}
\label{fig:agent}
\end{minipage}\hfill
\begin{minipage}{0.48\textwidth}
\centering
\includegraphics[width=\textwidth]{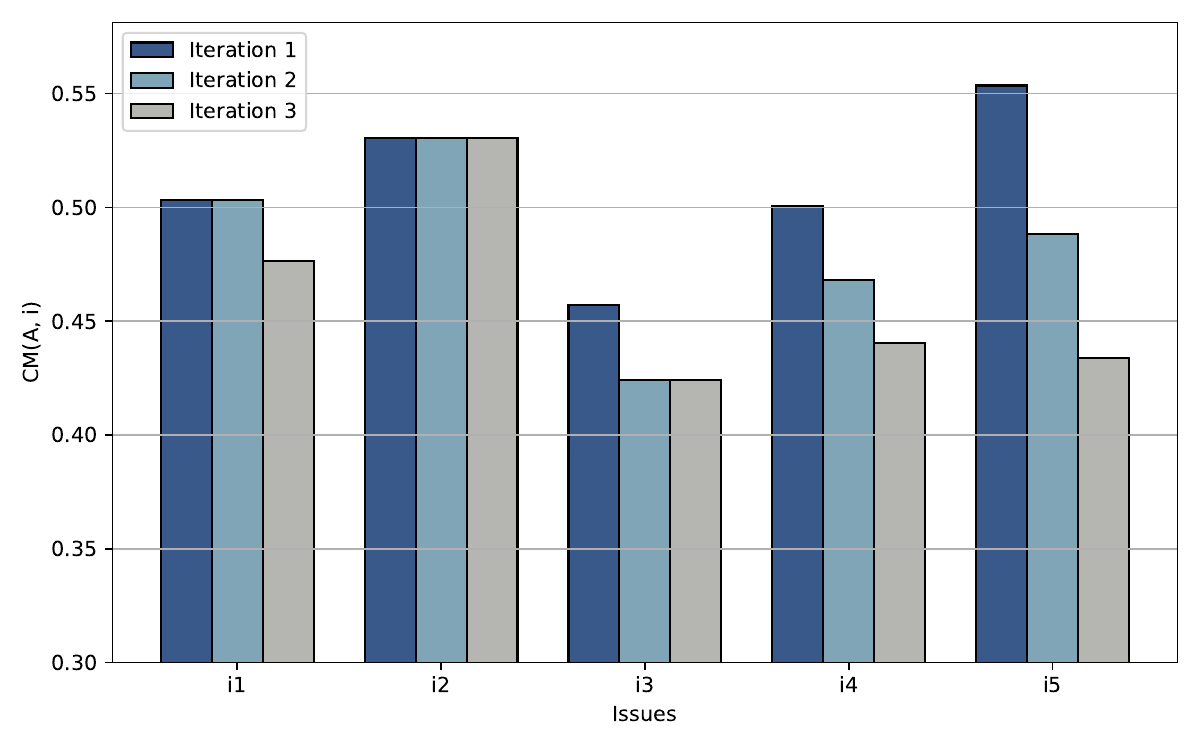}
\captionof{figure}{Changes of the conflict degree towards an individual issue.}
\label{fig:issue}
\end{minipage}
\end{figure}
\end{enumerate}		
\end{example}

\begin{figure}[htp]
\centering
\includegraphics[width=0.8\textwidth]{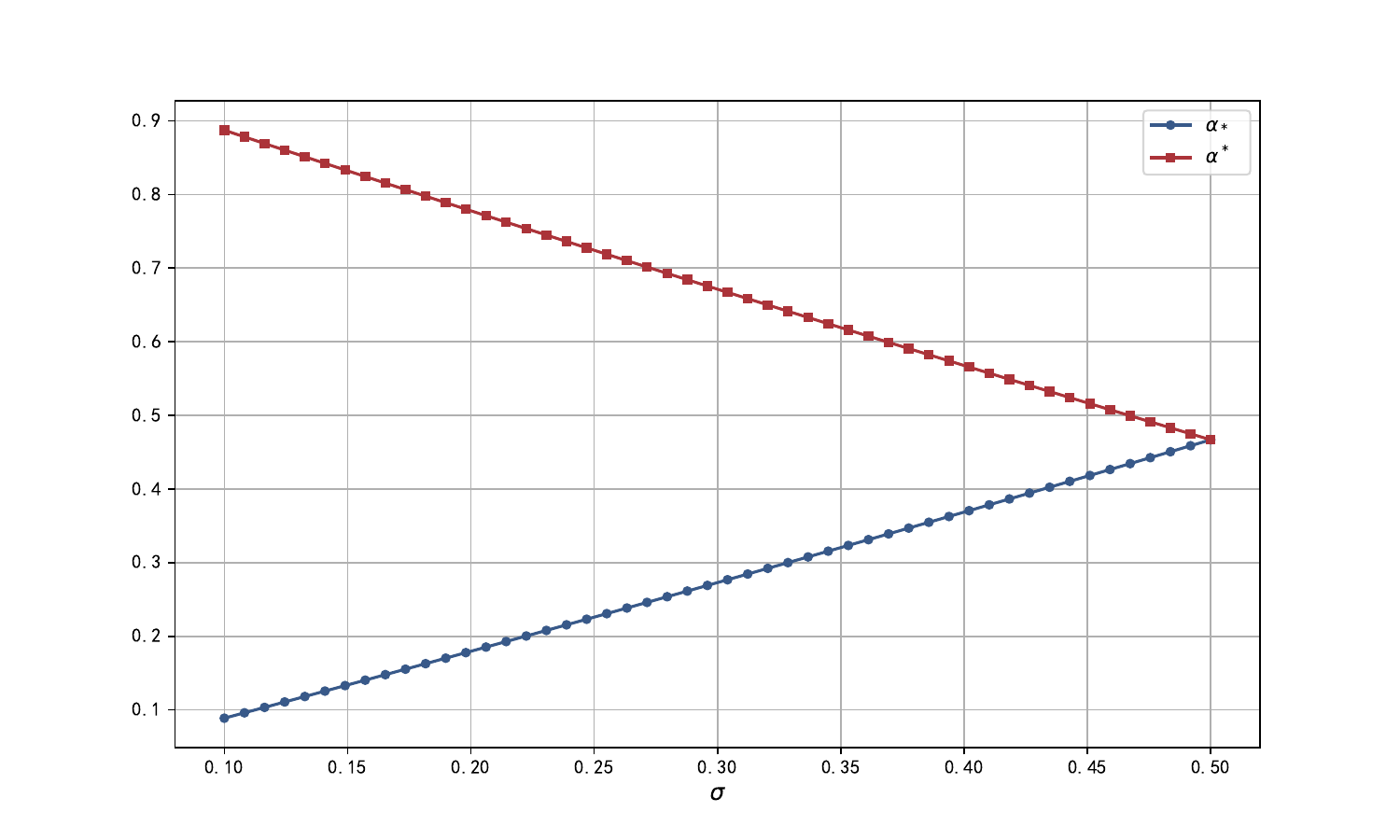}
\caption{The influence of parameter $\sigma$ on the thresholds $\alpha_*$ and $\alpha^*$.}
\label{fig:sensitivity}
\end{figure}

Finally, we compare the proposed model with existing approaches from three perspectives, namely, Agents' Attitudes, Threshold Calculation, and Conflict Resolution, and summarize the comparative results in Table~\ref{2} as follows.
\begin{enumerate}[label={(\arabic*)}]
	\item {\bf Agents' Attitudes:}
	Numerous studies~\cite{yao2019three,li2023three,xu2022selection, liu2025feasible, tong2021trust, gao2025three,lang2024trisection, Yang} model agents' attitudes using individual ratings on single issues, thereby overlooking heterogeneity in agents' interpretations of these ratings. To address this limitation, Hu's model~\cite{hu2025three} replaced ratings with preferences. However, due to the binary nature of classical preferences, it remains difficult to capture variations in the intensity of preferences shared by different agents. Lang's model~\cite{lang2025ija} and Shang's model~\cite{shang2025three} utilized multi-level preference and fuzzy preference to characterize agents' attitudes towards issue pairs, respectively.
	The proposed model employs intuitionistic fuzzy preferences to characterize agents' attitudes, enabling a more expressive representation of preferences from two complementary perspectives.
	
	\item  {\bf Threshold Calculation:}
	Two thresholds are commonly used to trisect the set of agent pairs. In existing studies \cite{yao2019three,li2023three, xu2022selection,  liu2025feasible, tong2021trust,Yang, hu2025three,lang2025ija,shang2025three}, thresholds are determined subjectively, whereas Gao's model~\cite{gao2025three} derives the thresholds by the given conflict function, and Lang's model~\cite{lang2024trisection} derives thresholds based on loss functions, which involve a relatively large number of parameters. The proposed model determines thresholds using relative loss functions derived from the proposed conflict functions, thereby reducing parameter complexity while maintaining objectivity.
	
	\item {\bf Conflict Resolution:}
	Conflict resolution is a central objective of three-way conflict analysis. However, several existing models, including those of Yao~\cite{yao2019three}, Lang~\cite{lang2024trisection}, and  and Hu~\cite{hu2025three}, do not explicitly address conflict resolution. Xu's model~\cite{xu2022selection} identifies multiple issues with high consistency as feasible strategies, while Li's model~\cite{li2023three} and Yang's model~\cite{Yang} focus on multiple issues with low conflict degrees. These approaches, however, do not adequately address issues characterized by both low consistency and high conflict. In contrast, Tong's model~\cite{tong2021trust} adopts a trust-based mechanism to adjust agents' ratings, and Gao's model~\cite{gao2025three} proposes preference adjustment methods for conflict resolution. Building upon these efforts, the proposed model introduces an iterative optimization-based approach that simultaneously accounts for both adjustment magnitude and conflict degree, thereby providing a more intuitive and effective framework for conflict resolution.
\end{enumerate}

\begin{sidewaystable}
\centering
\small
\setlength{\tabcolsep}{12pt}
\caption{Comparison of the proposed model and other models.}
\label{2}
				
\begin{threeparttable}
\begin{tabular}{llll}
\toprule
$\mathrm{Models}$ &
Agents'\ attitude &
$\mathrm{Thresholds}$ &
$\mathrm{Feasible\ strategy}$ \\
\midrule
Yao's model~\cite{yao2019three} &
$\mathrm{Three\text{-}valued,\ many\text{-}valued}$ &
$\mathrm{Subjectivity}$ &$\times$ \\
Li's\ model~\cite{li2023three} &
$\mathrm{q\text{-}rung\ orthopair\ fuzzy\ number}$ &
$\mathrm{Subjectivity}$ &
$\mathrm{Non-adjustment}$ \\
Xu's\ model~\cite{xu2022selection} &
$\mathrm{Three\text{-}valued,\ fuzzy\ number}$ &
$\mathrm{Subjectivity}$ &
$\mathrm{Non-adjustment}$ \\
Liu's model~\cite{liu2025feasible} &
$\mathrm{Three\text{-}valued}$ &
$\mathrm{Subjectivity}$ &
$\mathrm{Non-adjustment}$ \\
Tong's\ model~\cite{tong2021trust} &
$\mathrm{Pythagorean\ fuzzy\ number}$ &
$\mathrm{Subjectivity}$ &
$\mathrm{Adjustment}$ \\
Gao's model~\cite{gao2025three} &
$\mathrm{Interval\ set}$ &
$\mathrm{Objective}$ &
$\mathrm{Adjustment}$ \\
Lang's model~\cite{lang2024trisection} &
$\mathrm{Pythagorean\ fuzzy\ number}$ &
$\mathrm{Objective}$ &
$\times$ \\
Yang's model~\cite{Yang} &
$\mathrm{Various\ fuzzy\ numbers}$ &
$\mathrm{Subjectivity}$ &
$\mathrm{Non-adjustment}$ \\
Hu's model~\cite{hu2025three} &
$\mathrm{Preference}$ &$\mathrm{Subjectivity}$ &
$\times$ \\
Lang's model~\cite{lang2025ija} &
$\mathrm{Multi-level\ preference}$ &$\mathrm{Subjectivity}$ &
$\times$ \\
Shang's model~\cite{shang2025three} &
$\mathrm{Fuzzy\ preference}$ &$\mathrm{Subjectivity}$ &
$\times$ \\
The proposed model &
$\mathrm{Intuitionistic\ fuzzy\ preference}$ &
$\mathrm{Objective}$&
$\mathrm{Adjustment}$\\
\bottomrule
\end{tabular}
\begin{tablenotes}
\footnotesize
\item \textit{Note:} Static: consensus without rating changes; Dynamic: consensus through rating adjustment.
\end{tablenotes}
\end{threeparttable}
\end{sidewaystable}

\section{Conclusion}	\label{sec:Conslusion}
The characterization of conflicts and feasible strategies for conflict resolution are two important topics of three-way conflict analysis.
In this paper, we first introduced intuitionistic fuzzy preference-based conflict situations, enabling a more fine-grained representation of agents' attitudes towards issue pairs than that provided by classical preference-based conflict situations. Afterwards, we developed intuitionistic fuzzy preference-based conflict measures and corresponding three-way conflict analysis models to trisect agent pairs, the agent set, and the issue set. Meanwhile, a threshold calculation method based on conflict functions was proposed to support the implementation of three-way conflict analysis models. Additionally, adjustment mechanism-based feasible strategies for conflict resolution were formulated by jointly considering the adjustment magnitude and the conflict degree. Finally, we provided an illustrative example to demonstrate the validity and effectiveness of the proposed approach.

In the future, we will focus on developing more effective feasible strategies for conflict resolution within intuitionistic fuzzy preference-based conflict situations, along with efficient algorithms for computing such strategies. In addition, the proposed framework will be extended to hybrid preference-based conflict situations with more constraints. Further work will also consider incorporating the weights of agents and issues to design feasible strategies for conflict resolution in complex conflict situation tables.

%
%
%
%
%

\section*{Acknowledgments}
This work was supported by National Natural Science Foundation of China (No. 12501630, 62076040, 12471431), the Scientific Research Fund of Hunan Provincial Education Department (No. 22A0233), a Discovery Grant from the Natural Sciences and Engineering Research Council of Canada (No. RGPIN-2024-04410), and Shandong Provincial Natural Science Foundation (No. ZR2022QG033).

\bibliographystyle{unsrt} 
\bibliography{Manuscript}
\end{sloppypar}
\end{document}